\documentclass[10pt]{article} 
\usepackage[accepted]{tmlr}
\usepackage{amsthm}     

\usepackage{amsmath,amsfonts,bm}

\def\eqref#1{equation~\ref{#1}}

\def\1{\bm{1}}

\newtheorem{lemma}{Lemma}
\newtheorem{theorem}{Theorem}

\newtheorem{proposition}{Proposition}

\newtheorem{remark}{Remark}
\newtheorem{assumption}{Assumption}

\def\mE{{\bm{E}}}
\def\mF{{\bm{F}}}
\def\mG{{\bm{G}}}

\def\mL{{\bm{L}}}

\def\mR{{\bm{R}}}

\def\mX{{\bm{X}}}

\DeclareMathAlphabet{\mathsfit}{\encodingdefault}{\sfdefault}{m}{sl}
\SetMathAlphabet{\mathsfit}{bold}{\encodingdefault}{\sfdefault}{bx}{n}

\newcommand{\E}{\mathbb{E}}

\newcommand{\R}{\mathbb{R}}

\DeclareMathOperator*{\argmax}{arg\,max}
\DeclareMathOperator*{\argmin}{arg\,min}

\usepackage{multirow}
\def\wh{\widehat}
\def\wt{\widetilde}

\def\pr{\hbox{pr}}

\def\trans{^{\rm T}}

\def\n{\nonumber}

\def\f{{\boldsymbol{f}}}

\def\I{{\bf I}}
\def\R{{\bf R}}

\def\m{{\bf m}}

\def\X{{\bf X}}
\def\x{{\bf x}}

\def\d{{\bf d}}

\def\Z{{\bf Z}}

\def\p{{\bf p}}
\def\mL{{\mathcal{L}}}
\def\bL{{\mathfrak{L}}}

\def\E{{\mathrm{E}}}

\def\mR{\mathcal{R}}
\def\mX{\mathcal{X}}
\def\mE{\mathcal{E}}
\def\mF{\mathcal{F}}
\def\mG{\mathcal{G}}

\def\bmu{\boldsymbol\mu}

\def\bb{{\boldsymbol\beta}}

\def\bmu{\boldsymbol\mu}

\def\0{{\bf 0}}
\def\trans{^{\rm T}}
\def\pr{\hbox{pr}}
\def\wh{\widehat}
\def\wt{\widetilde}

\def\argmin{\mbox{argmin}}
\def\argmax{\mbox{argmax}}

\def\log{{\rm log}}

\def\bse{\begin{eqnarray*}}
\def\ese{\end{eqnarray*}}
\def\be{\begin{eqnarray}}
\def\ee{\end{eqnarray}}
\def\bsq{\begin{equation*}}
\def\esq{\end{equation*}}
\def\bq{\begin{equation}}
\def\eq{\end{equation}}

\def\sumIP1{\sum_{i=1, i\in P_1}^N}

\usepackage{hyperref}
\usepackage{url}

\usepackage{graphicx}  
\usepackage{graphics} 
\usepackage{amsmath}  
\usepackage{amsfonts}
\usepackage{amssymb}
\usepackage{makecell}
\usepackage{booktabs} 
\usepackage{algorithm}
\usepackage{algpseudocode}
\usepackage{enumitem}
\usepackage{caption}
\usepackage{subcaption}

\title{Unsupervised Domain Adaptation for Binary Classification with an Unobservable Source Subpopulation}

\author{\name Chao Ying \email chao.ying@wisc.edu \\
	\addr Departments of Statistics and of Biostatistics \& Medical Informatics \\
	University of Wisconsin-Madison
	\AND
	\name Jun Jin \email jjin2@hfhs.org \\
	\addr Public Health Sciences\\
	Henry Ford Health
	\AND
	\name Haotian Zhang \email haotiangeek@gmail.com \\
	\addr School of Computing\\
		University of Connecticut
	\AND
	\name Qinglong Tian \email qinglong.tian@uwaterloo.ca \\
	\addr Department of Statistics and Actuarial Science\\
	University of Waterloo
		\AND
	\name Yanyuan Ma \email yzm63@psu.edu \\
	\addr Department of Statistics\\
	Pennsylvania State University
		\AND
	\name Sharon Li \email sharonli@cs.wisc.edu \\
	\addr Department of Computer Sciences\\
	University of Wisconsin-Madison
		\AND
	\name Jiwei Zhao \email jiwei.zhao@wisc.edu \\
	\addr Departments of Statistics and of Biostatistics \& Medical Informatics\\
	University of Wisconsin-Madison
	}

\def\month{04}  
\def\year{2026} 
\def\openreview{\url{https://openreview.net/forum?id=aOKcvMt8xE}} 

\begin{document}

\maketitle

\begin{abstract}
We study an unsupervised domain adaptation problem where the source domain consists of subpopulations defined by the binary label $Y$ and a binary background (or environment) $A$. 
We focus on a challenging setting in which one such subpopulation in the source domain is unobservable. Naively ignoring this unobserved group can result in biased estimates and degraded predictive performance. 
Despite this structured missingness, we show that the prediction in the target domain can still be recovered.
Specifically, we rigorously derive both background-specific and overall predictive probabilities for the target domain. 
For practical implementation, we propose the distribution matching method to estimate the subpopulation proportions. 
We provide theoretical guarantees for the asymptotic behavior of our estimator, and establish an upper bound on the prediction error. 
Experiments on both synthetic and real-world datasets show that our method outperforms the naive benchmarks that do not account for this unobservable source subpopulation properly.
\end{abstract}
{\bf Key Words:} Unsupervised domain adaptation, Structured missingness, Distribution matching

\section{Introduction}\label{sec:intro}
	
	Unsupervised domain adaptation (UDA) \citep{kouw2019review} addresses the challenge of transferring predictive models from a labeled source domain to an unlabeled target domain under distributional shifts 	\citep{koh2021wilds,sagawaextending}.
	In this area, research methods aim to reduce domain discrepancy by aligning feature distributions, using statistical measures such as maximum mean discrepancy (MMD) \citep{tzeng2014deep} and higher-order moment matching (HoMM) \citep{chen2020homm}.
	Deep adaptation frameworks, such as deep adaptation network (DAN) \citep{long2015learning} and domain-adversarial neural network (DANN) \citep{ganin2016domain}, are also popularly used due to their strong empirical performance.
	There are also other approaches that integrate reconstruction objectives to disentangle domain-invariant and domain-specific components \citep{ghifary2016deep}. These approaches often assume access to a representative and diverse set of source examples. However, real-world datasets may violate this assumption in systematic and non-random ways.

In many applications, the goal is to build a target-domain prediction model to predict a binary label $Y$ that identifies a specific object in an image (e.g., classifying a waterbird versus a landbird in the Waterbirds dataset). 
Additionally, the image typically contains contextual attributes represented by a binary background or environment variable $A$, such as whether the background consists of water or land.
In this work, we focus on a more challenging and practically relevant UDA setting where a \emph{structured subpopulation is entirely missing from the source domain}. 
Specifically, we study the case where one subpopulation, defined by a particular combination of $Y$ and $A$, is unobserved in the source. 
This structured missingness is not merely a sampling artifact, but often reflects real-world constraints in data collection. For instance, in the widely studied Waterbirds dataset~\citep{sagawa2019distributionally}, waterbirds ($Y = 1$) photographed in water environments ($A = 1$) can be rare or entirely absent due to the difficulty of capturing such images in the wild. This issue arises in many other disciplines as well. In healthcare, certain patient subgroups, defined jointly by disease status and demographics, may be underrepresented or absent in historical datasets due to restrictive inclusion criteria or changes in clinical practice over time. When such models are applied to broader populations, unobserved subgroups can suffer from systematic mispredictions. 
This structured missingness \citep{mitra2023learning} fundamentally changes some statistical properties when comparing the source and target domains, and, if unaddressed, can lead to severely biased estimation and unreliable prediction in the target domain. These structured gaps pose new challenges that are not adequately addressed by conventional UDA techniques, which motivates our work.

To tackle this challenge, we develop a theoretical framework that accounts for the structured absence of a subpopulation, such as $(Y = 1, A = 1)$, in the source domain. Our key idea is to model how prediction in the target domain can still be recovered by relating it to the observable parts of the source and target data. Under a mild assumption that the distribution of features $X$ given $(Y, A)$ stays the same across domains, we derive closed-form expressions for making accurate predictions in the target domain. These expressions depend on the proportions of different subgroups in the target, which are unknown. To estimate them, we propose a practical method based on distribution matching that avoids modeling complex feature distributions directly. Specifically, we frame the problem as estimating finite-dimensional mixture proportions under structured conditional invariance, and propose a KL-divergence-based objective that can be optimized using only observable quantities.  
We also provide theoretical guarantees, showing that our approach yields statistically consistent estimates and deriving upper bounds on the prediction error of the resulting target-domain classifiers.
Overall, our framework provides the first rigorous characterization of model adaptation under structured subpopulation absence, and
enables robust domain adaptation in such a challenging scenario. 

We validate our approach through experiments on both synthetic and real-world datasets. We simulate domain adaptation scenarios where one subpopulation is systematically excluded from the source data and evaluate our method against baseline approaches that do not account for this missing group properly. Across a range of settings, our method consistently achieves higher accuracy and F1 scores. 
These results highlight the practical value of explicitly modeling structured missingness and demonstrate that our approach leads to more reliable predictions in the target domain. To summarize, this paper makes the following novel contributions:
\begin{itemize}
    \item We consider a new unsupervised domain adaptation setting where an entire label-background subpopulation is missing from the source domain, a scenario motivated by real-world data collection constraints.
    
    \item We develop a theoretical framework that enables accurate prediction in the target domain by estimating subpopulation proportions through distribution matching, and we provide rigorous guarantees and error bounds for our method.
    
    \item We demonstrate the effectiveness of our approach on both synthetic and real-world datasets. Our method outperforms standard baselines that ignore structured missingness, particularly in recovering performance on the unobserved subpopulation.
\end{itemize}

	\section{Related Work}
	
	\paragraph{Out-of-distribution (OOD) generalization}
	OOD generalization refers to the ability of a prediction model to perform well on test data drawn from a distribution that differs from the training data. 
    In our context, the subpopulation $(Y=1,A=1)$ in the target can be regarded as the OOD data while the other three subpopulations are in-distribution data.
For a comprehensive overview of OOD generalization, we refer the readers to the excellent survey \citep{liu2021towards}, which reviewed real-world datasets, evaluation protocols, and key challenges in this area.
	In the OOD generalization literature, different methods were proposed with different emphases: \cite{arjovsky2019invariant} emphasized the need to minimize invariant risk across different environments to ensure consistent model performance, whereas
	\cite{sagawa2019distributionally} underscored the importance of distributionally robust optimization (DRO) and various regularization techniques  
    in reducing performance disparities across subgroups.
	In addition, \cite{bahng2020learning} introduced adversarial training as a method for learning de-biased representations, which is critical for promoting fairness in machine learning models, and \cite{sohoni2020no} examined the issue of robustness in classification tasks involving coarse classes that contain finer subclasses, enhancing model performance across all subclasses. 

	\paragraph{OOD detection}
	OOD detection is the task of identifying inputs at test time that do not come from the same distribution as the training data.
	Its goal is to prevent a model from making confident but incorrect predictions on unfamiliar or anomalous inputs by flagging them as OOD.
	There are a variety of techniques developed for OOD detection in the literature. 
	For example, \cite{hendrycks2017baseline} introduced a simple yet effective method for detecting both misclassified and OOD inputs in neural networks. 
	\citet{liang2018enhancing} (ODIN) proposed an improved method for detecting OOD inputs by applying temperature scaling to the softmax outputs and adding small input perturbations during inference. 
	ODIN significantly outperformed previous baseline methods, including the maximum softmax probability approach, and set a new standard for OOD detection in classification tasks.
	Other techniques include but not limited to, outlier exposure \citep{hendrycks2018deep, papadopoulos2021outlier}, ConfGAN \citep{sricharan2018building} and OodGAN \citep{marek2021oodgan}.
	In addition, \cite{fort2021exploring} provided an extensive empirical study of OOD detection methods across a wide range of datasets, architectures, and training regimes.

	\paragraph{Spurious correlation}
	Spurious correlation is a major obstacle to OOD generalization, where models often rely on non-causal features that can degrade performance, particularly when these correlations do not generalize across domains. 
	For example, a model trained to classify cows might rely on green pastures (background) instead of the cow itself. On a desert background, it fails.
	This is also the case in the Waterbirds dataset where the spurious correlation exists between label $Y$ and background $A$. Different learning strategies were proposed to discover and mitigate the impact of spurious correlation on model performance, as well as to improve model robustness.
	For example, \cite{wu2023discover} introduced an attention-based approach to automatically identify spurious concepts and apply adversarial training to reduce reliance on them. 
    Another approach proposed by \cite{kumar2023causal}
	used causal regularization to detect and discourage spurious dependencies, allowing for scalable robustness across shifts.
	In addition, \cite{sagawa2020investigation} investigated why overparameterization exacerbates spurious correlations, and \cite{kirichenko2022last} found that retraining only the final layer on a small, balanced dataset can restore robustness against spurious correlations.
	Also, \cite{wang2024effect} developed a theoretical model to analyze the influence of spurious correlation strength, sample size, and feature noise on learning. 
	Spurious correlations were also investigated in feature learning \citep{izmailov2022feature, qiu2024complexity}, reinforcement learning \citep{ding2023seeing}, OOD detection \citep{ming2022impact}, and text classification \citep{wang2020identifying}.
	One can also resort to a comprehensive survey paper \citep{ye2024spurious} on this topic.

\paragraph{Imbalanced classification and few/zero-shot learning}
		In imbalanced classificaton, all classes are observed in the training data but appear with highly unequal frequencies, and many methods focus on reweighting or resampling strategies to improve performance, particularly on minority classes. Few-shot learning \citep{wang2020generalizing}, such as one-shot learning \citep{li2006one}, refers to a learning paradigm in which a model learns from a very small number of labeled examples and then generalizes to novel classes. Unlike traditional supervised learning that requires large amounts of data, few-shot learning leverages a limited number of examples together with task-specific prior knowledge or structural information. 
		Zero-shot learning \citep{wang2019survey} addresses tasks in which no labeled examples for the target classes are available during training, typically relying on auxiliary information such as semantic attributes or textual descriptions to relate seen and unseen classes.
		In contrast to these settings, our work considers a different challenge: a structured subgroup defined by a specific combination of $(Y,A)$ is completely absent from the source domain. 
		This missing-subgroup setting introduces a distinct type of distribution shift that cannot be addressed directly by existing methods designed for imbalanced classification or few/zero-shot learning.

\paragraph{Adversatial domain adaptation}
		Adversarial domain adaptation methods, such as DANN~\citep{ganin2016domain} and ADDA~\citep{tzeng2017adversarial}, aim to align the source and target feature distributions using adversarial training. 
		However, in the setting with structured missingness that we consider, these methods can fail because the unobserved target subgroup may be incorrectly aligned to a visible source subgroup, a phenomenon termed as ``collapse''. 
		This occurs because adversarial alignment enforces marginal distribution matching without modeling hidden subpopulation structure, which can lead to biased predictions. 
		Differently, our framework explicitly accounts for the unobservable source subpopulation, avoiding the collapse issue and providing more reliable adaptation in such scenarios.

	\section{Problem Setup and Notation}\label{sec:prelim}
	
	In our UDA setting,
	$Y \in \{0,1\}$ denotes the binary label, which is observed in the source domain but not in the target. 
	Let $A \in \{0,1\}$ be a binary background or environment variable and $\X \in \R^q$ a vector of all other attributes. 
	Let $R \in \{0,1\}$ be a domain indicator, with $R = 1$ corresponding to the source and $R = 0$ to the target. 
    In our notation, we consistently use the order of $(R, Y, A)$ for 
	indicator function $I_{\{\cdot\}}$, sample size $n_{\{\cdot\}}$, and population probability $p_{\{\cdot\}}$.
    
	We define $\pi=\pr(R=1)$. For $y=1,0$, $a=1,0$, we define $\alpha_{ya}=\pr(Y=y, A=a \mid R=1)$, and $\beta_{ya}=\pr(Y=y, A=a \mid R=0)$.
	For clarity, the total source sample size is $n_1 = n_{101} + n_{110} + n_{100}$, and the target sample size is $n_0 = n_{0\cdot1} + n_{0\cdot0}$, so that the total sample size is $n = n_1 + n_0$. 
	Table~\ref{tb:data} summarizes the observed data structure and key notation. 
	\begin{table}[!htbp]
		\centering
		\caption{Data structure and key notation used throughout the paper.}\label{tb:data}
		\resizebox{\textwidth}{!}{
			\begin{tabular}{cccccccc}
				\hline
				& $R$ & $Y$ & $A$ & $\X$ & Sample Size & Proportion & Probability\\
				\hline
				\multirow{3}{*}{Source} 
				&\multirow{3}{*}{\makecell{1\\1\\1}} & \multirow{3}{*}{\makecell{0\\1\\0}} & \multirow{3}{*}{\makecell{1\\0\\0}} & \multirow{3}{*}{\makecell{\checkmark\\\checkmark\\\checkmark}} 
				& \multirow{3}{*}{\makecell{$n_{101}$\\$n_{110}$\\$n_{100}$}} & \multirow{3}{*}{\makecell{$p_{101}=\alpha_{01} \pi$\\$p_{110}=\alpha_{10} \pi$\\$p_{100}=\alpha_{00} \pi$}} 
				& $\xi_1(\x)=\pr(Y=1 \mid \X=\x, A=1, R=1)\equiv 0$\\
				& &  &  & & & &$\xi_0(\x)=\pr(Y=1 \mid \X=\x, A=0, R=1)$ \\
				& &  &  & & & &$\xi(\x)=\pr(Y=1 \mid \X=\x, R=1)$ \\
				\hline
				\multirow{4}{*}{Target} 
				&0&?&1&\checkmark
				& \multirow{2}{*}{$n_{0\cdot1}$} & \multirow{2}{*}{$p_{0\cdot1}=(\beta_{11}+\beta_{01})(1-\pi)$} &  \multirow{4}{*}{\makecell{$\eta_1(\x)=\pr(Y=1 \mid \X=\x, A=1, R=0)$\\$\eta_0(\x)=\pr(Y=1 \mid \X=\x, A=0, R=0)$ \\$\eta(\x)=\pr(Y=1 \mid \X=\x, R=0)$}} \\
				&0 &? &1 &\checkmark & & & \\
				\cmidrule(r){2-7}
				&0 &? &0 &\checkmark & \multirow{2}{*}{$n_{0\cdot0}$} &\multirow{2}{*}{$p_{0\cdot0}=(\beta_{10}+\beta_{00})(1-\pi)$} &\\
				&0 &? &0 &\checkmark & & &\\
				\hline
		\end{tabular}}
	\end{table}

    In our context, we have
	$\alpha_{10} + \alpha_{01} + \alpha_{00} = 1, \quad \alpha_{11} = 0, \quad \text{and} \quad 0 < \alpha_{10}, \alpha_{01}, \alpha_{00} < 1$.
	The parameters can be consistently estimated by
	\be\label{est:alpha}
	\widehat{\alpha}_{10} = n_{110}/n_1, \quad
	\widehat{\alpha}_{01} = n_{101}/n_1, \quad
	\widehat{\alpha}_{00} = n_{100}/n_1, \quad
	\widehat{\pi} = n_1/n.
	\ee

    More formally, $\alpha_{11}=0$ is the following structured missingness condition:
	\begin{equation}\label{eq:ood}
		\pr(Y=1, A=1 \mid R=1) = 0.
	\end{equation}
	Note that this assumption is made without loss of generality, as alternative combinations, such as $(Y=0, A=1)$, $(Y=1, A=0)$, or $(Y=0, A=0)$, can be similarly assumed to have zero probability. 
	
    To characterize the distributional connection between the two domains, we impose a structured conditional invariance assumption:
	\be\label{eq:invariance}
	p(\X \mid Y, A, R=1) = p(\X \mid Y, A, R=0) = p(\X\mid Y,A) \equiv p_{ya}(\X),
	\ee
	that is, the conditional distribution of features $\X$ given $(Y,A)$ remains the same across domains.
	This can be regarded as a conditional version, or, more nuanced version, of label shift where the marginal distribution of labels (now, the combination of both label and background) varies across domains (e.g., \citealp{du2014semi, garg2020unified, iyer2014maximum, lipton2018detecting, nguyen2016continuous, tasche2017fisher, tian2023elsa, zhang2013domain, lee2025doubly, lee2025efficient}). 
	This type of invariance assumptions was also postulated in other contexts such as in causal inference; see, e.g., \cite{peters2016causal}.
	It indicates, conditional on background $A$, the label shift assumption holds. 
	It is equivalent to $p(R|\X,Y,A)=p(R|Y,A)$, the independence between $R$ and $\X$, conditional on $(Y, A)$.
	In practice, this assumption may be suitable in many applications. Below we give two examples to illustrate the rationality of this assumption. For instance, we aim to predict user clicks on advertisements for a new batch of users (target domain, $R=0$) using historical data (source domain, $R=1$). Conditional on the advertisement type $A$ and whether the user clicks $Y$, the distribution of browsing behavior features $\X$ is assumed to remain stable across time periods. This is because user clicks are fundamentally determined by ad content and user interests, not by the time period in which data are collected. As another example, suppose we have datasets from two hospitals ($R=1$ indicates the source hospital and $R=0$ indicates the target hospital). Here, $\X$ represents imaging features, $Y$ is the disease type, and $A$ denotes patient attributes such as gender or age group. Then, conditional on the disease type $Y$ and demographic attributes $A$, the distribution of imaging features $\X$ is expected to remain the same across hospitals. This is because imaging characteristics for a given disease and demographic group are not systematically altered by the hospital. The main difference between hospitals lies in sampling proportions rather than in conditional distributions.

	This framework captures real-world scenarios in which a certain label-background subpopulation is absent from the source domain. 
	For example, in the Waterbirds dataset, waterbirds on water backgrounds (label $Y = 1$, background $A = 1$) are rarely observed, or even completely absent, in the training set, making the adaptation to target domains particularly challenging.
    For illustration purposes, Table~\ref{tb:examples} below shows the three observed subpopulations in the source as well as the four subpopulations in the target in two real-world datasets.
	\begin{table}[!htbp]
		\centering
		\caption{
			Illustrations in Waterbirds and CelebA datasets. Note that the $(Y=1,A=1)$ combination does not exist in the source domain but does in the target domain.
			}
		\label{tb:examples}
		\resizebox{\textwidth}{!}{
			\renewcommand{\arraystretch}{1.5}
			\setlength{\tabcolsep}{8pt}
			\begin{tabular}{ccc}
				\toprule
				\textbf{Dataset} & \textbf{Source Data} & \textbf{Target Data} \\
                     $(Y,A)$&$(0,1)$\;\qquad\qquad \qquad\qquad  $(1,0)$\;\qquad\qquad\qquad\qquad $(0,0)$&$(1,1)$\qquad\qquad\qquad\qquad $(0,1)$\qquad\quad\qquad \qquad\qquad $(1,0)$\qquad\qquad\qquad\qquad$(0,0)$\\
				\midrule
				\multirow{2}{*}{Waterbirds} 
				& 
				\begin{tabular}{cccc}
					\makecell{\textcolor{red}{Y=0:Landbird}\\ A=1:Water background\\ \includegraphics[width=2cm]{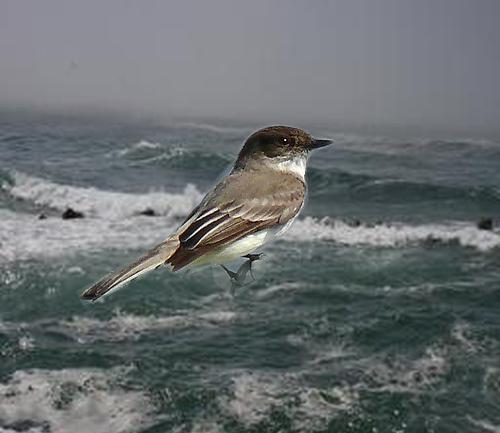}} &
					\makecell{\textcolor{green}{Y=1:Waterbird}\\ \textcolor{blue}{A=0:Land background}\\ \includegraphics[width=2cm]{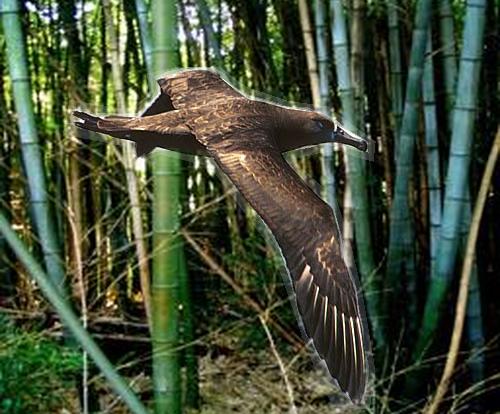}} &
                    	\makecell{\textcolor{red}{Y=0:Landbird}\\ \textcolor{blue}{A=0:Land background}\\ \includegraphics[width=2cm]{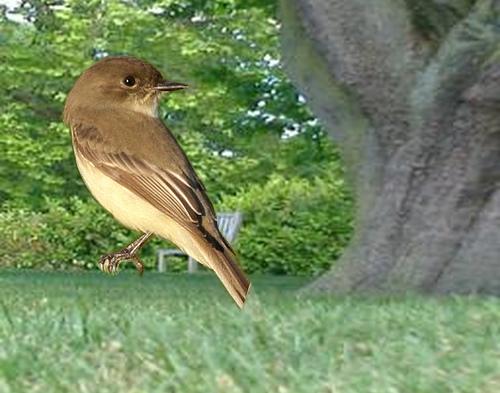}} \\
				\end{tabular}
				&
				\begin{tabular}{cccc}
					\makecell{\textcolor{green}{Y=1:Waterbird}\\ A=1:Water background\\ \includegraphics[width=2cm]{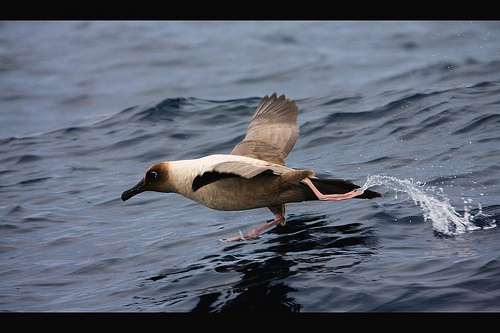}} &
					\makecell{\textcolor{red}{Y=0:Landbird}\\ A=1:Water background\\ \includegraphics[width=2cm]{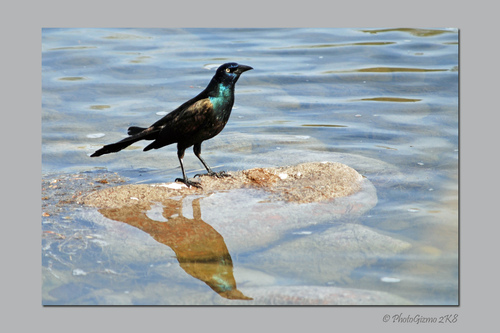}} &
					\makecell{\textcolor{green}{Y=1:Waterbird}\\ \textcolor{blue}{A=0:Land background}\\ \includegraphics[width=2cm]{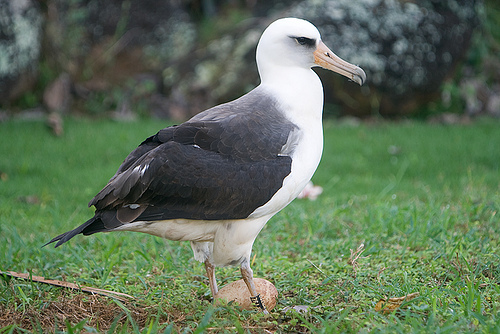}} &
                    \makecell{\textcolor{red}{Y=0:Landbird}\\ \textcolor{blue}{A=0:Land background}\\ \includegraphics[width=2cm]{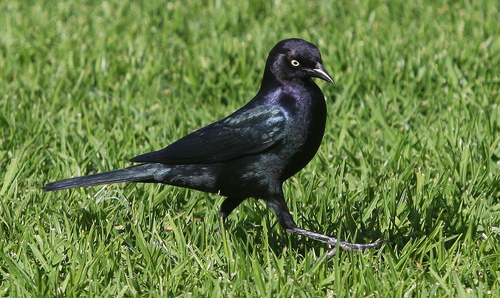}} \\
				\end{tabular}
				\\
				\midrule
				\multirow{2}{*}{CelebA} 
				& 
				\begin{tabular}{cccc}
					\makecell{\textcolor{red}{Y=0:Blond hair}\\    A=1:Male \\ \includegraphics[width=2cm,height=2cm]{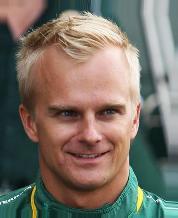}}& \qquad\quad
                    \makecell{\textcolor{green}{Y=1:Dark hair}\\ \textcolor{blue}{A=0:Female} \\ \includegraphics[width=2cm,height=2cm]{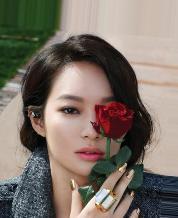}}&\qquad
					\makecell{\textcolor{red}{Y=0: Blond hair}\\ \textcolor{blue}{A=0:Female }\\ \includegraphics[width=2cm,height=2cm]{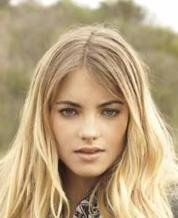}} \\
				\end{tabular}
				&
                \begin{tabular}{cccc}
					\makecell{\textcolor{green}{Y=1:Dark hair}\\ A=1:Male \\ \includegraphics[width=2cm,height=2cm]{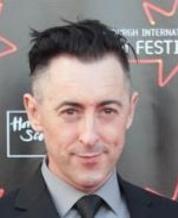}}\ \ \ \ \ & \ \ \ \ \ \ \
					\makecell{\textcolor{red}{Y=0:Blond hair}\\  A=1:Male\\ \includegraphics[width=2cm,height=2cm]{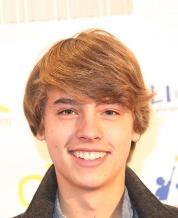}}\ \ \ \ \ \ \ & \ \ \ \ \ \ \
                    \makecell{\textcolor{green}{Y=1: Dark hair}\\ \textcolor{blue}{A= 0:Female}\\ \includegraphics[width=2cm,height=2cm]{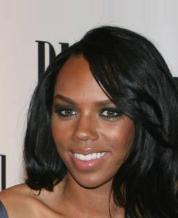}}\ \ \ \ \ \ \ & \ \ \ \ \ \ \
					\makecell{\textcolor{red}{Y=0:Blond hair}\\ \textcolor{blue}{A=0:Female} \\ \includegraphics[width=2cm,height=2cm]{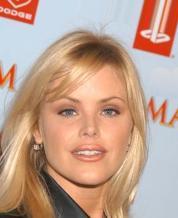}} \\
				\end{tabular}\\
				\bottomrule
		\end{tabular}}
	\end{table}

	\section{Proposed Methodology}\label{sec:method}

    Our goal in this work is to correctly identify and successfully implement, under our UDA setting, the two background-specific predictive probabilities $\eta_1(\x)$ and $\eta_0(\x)$ and the overall predictive probability $\eta(\x)$, in the target domain. 
    All of the three probabilities were precisely defined in Table~\ref{tb:data}.

    \subsection{The naive benchmarks}\label{sec:naive}
    
    We consider two types of naive benchmarks.
    The first, termed \textbf{Naive1} throughout, is to blindly apply the three source domain predictive probabilities $\xi_1(\x)$, $\xi_0(\x)$, and $\xi(\x)$ to the target, ignoring the structured missingness issue in this problem.
    More precisely, the first naive benchmark \emph{incorrectly} treats $\xi_1(\x)$, $\xi_0(\x)$, and $\xi(\x)$ as $\eta_1(\x)$, $\eta_0(\x)$, and $\eta(\x)$, respectively.
    Note that, even though $\xi_1(\x)\equiv 0$ as indicated in Table~\ref{tb:data}, one can still use the observable data to implement the overall source domain predictive probability $\xi(\x)$ and the other background-specific predictive probability $\xi_0(\x)$ as
    \begin{align}
    	\xi(\x) = \pr(Y = 1 \mid \x, R = 1), \mbox{ and } \xi_0(\x) = \pr(Y = 1 \mid \x, R = 1, A = 0).\label{eq:xiandxi0}
    \end{align}
    The second type, termed \textbf{Naive2}, is to ignore the environment $A$ and simply impose a label shift assumption between the source and target domains.
    That is, one \emph{incorrectly} specifies the following $\gamma(\x)$ as the overall target domain predictive probability $\eta(\x)$.
    After some simple algebra, one can compute
    \begin{align}\label{eq:naive2}
    	\gamma(\x) = \frac{\frac{\beta_{11}+\beta_{10}}{\alpha_{11}+\alpha_{10}}\xi(\x)}{ \frac{\beta_{11}+\beta_{10}}{\alpha_{11}+\alpha_{10}}\xi(\x) + \frac{\beta_{01}+\beta_{00}}{\alpha_{01}+\alpha_{00}}{\{1-\xi(\x)\}}}.
    \end{align}

	\subsection{Model adaptation from source to target}\label{sec:target}
	
	The most challenging aspect of this work is to adapt the model for the $A=1$ background since the component $(Y=1,A=1)$ is entirely absent in the source.	Nevertheless, we can still \emph{correctly} derive the three predictive probabilities for the target domain, as shown below.
	\begin{proposition}\label{pro:relation}
    Define conditional probabilities $\tau_0(\x) = \pr(A=1 \mid \X=\x, R=0)$ and 
    		\begin{align}
			\kappa(\x) = \pr(R=1\mid \x,A=1),\label{eq:kappa}
		\end{align}
        both of which can be implemented using the observed data in our UDA setting.
		Then the three predictive probabilities in the target domain are given by:
\begin{equation} \label{eq:objection}
\begin{aligned}
		\eta_1(\x) 
		&= 1 - \frac{\beta_{01}}{\alpha_{01}} \cdot \frac{1 - \pi}{\pi} \cdot \frac{\kappa(\x)}{1-\kappa(\x)}, \quad
		\eta_0(\x) =
		\frac{
			\frac{\beta_{10}}{\alpha_{10}} \xi_0(\x)
		}{
			\frac{\beta_{10}}{\alpha_{10}} \xi_0(\x)
			+
			\frac{\beta_{00}}{\alpha_{00}} \{1 - \xi_0(\x)\}
		}, \mbox{ and }\\
		\eta(\x) &= \eta_1(\x) \tau_0(\x) + \eta_0(\x) \{1 - \tau_0(\x)\}.
\end{aligned}
\end{equation}
	\end{proposition}
	The proof of this result is provided in Appendix~\ref{sec:supp:method}.
	Proposition~\ref{pro:relation} illustrated that, in general, the naive benchmarks presented in Section~\ref{sec:naive} fail.
	There are no explicit relations between $\eta_1(\x)$ and $\xi_1(\x)\equiv 0$ or between $\eta(\x)$ and $\xi(\x)$.
	For the relation between $\eta_0(\x)$ and $\xi_0(\x)$, they coincide only in the special case that $\beta_{10}/\alpha_{10}=\beta_{00}/\alpha_{00}$, which corresponds to a proportionality condition between the class-conditional densities across domains. Outside of this narrow scenario, the naive approach systematically misestimates the target posterior, leading to biased predictions.
	
	This result also implies that model adaptation fundamentally relies on estimating the proportions of key subgroups in the target population. In particular, for individuals with $A=1$, one only needs to estimate $\beta_{01}$, while for those with $A=0$, it suffices to estimate the ratio $\beta_{10}/\beta_{00}$.
	Denote $\bb = (\beta_{10}, \beta_{00})\trans$. It can be seen that, accurate estimation of the parameter $\bb$ in the target domain enables valid model adaptation across domains.
    Before developing methods for estimating $\bb$ in Section~\ref{sec:estbeta}, we first present some model identification considerations.
	
	\subsection{Model identification considerations}\label{sec:model}
	
	The identifiability structure of our problem closely resembles that of the \textit{open set label shift} (OSLS) framework \citep{garg2022domain}. 
	Note that our target distribution 
	consists of a mixture over four joint distributions: $\pr(Y=1, A=1)$, $\pr(Y=1, A=0)$, $\pr(Y=0, A=0)$, and $\pr(Y=0, A=1)$. 
	By treating the joint label $(Y, A)$ as the response, this setting can be viewed as a special case of the OSLS framework. 
	However, our setup is considerably simpler due to the availability of the auxiliary variable $A$ in the target domain.
	As a result, we can restrict attention to the subset $A = 1$, thereby discarding the $A = 0$ portion of the distribution. 
	This reduction simplifies the problem to recovering $\pr(Y=1, A=1)$ from a mixture of $\pr(Y=1, A=1) $ and $\pr(Y=0, A=1) $, given direct access to $\pr(Y=0, A=1)$. 
	This is a canonical \textit{positive-unlabeled (PU)} learning problem. 
	Identifiability in this setting is governed by the standard \textit{anchor set condition} (see Definition~8 of \cite{ramaswamy2016mixture}): there exists a measurable subset $\x_{\mathrm{anchor}}\in \mX $ such that
	\bse
	p(\X\in\x_{\mathrm{anchor}}|Y=0,A=1)>0 \quad \text{and} \quad \frac{p(\X\in\x_{\mathrm{anchor}}|Y=1,A=1)}{p(\X\in\x_{\mathrm{anchor}}|Y=0,A=1)}=0.
	\ese
	This condition ensures that the positive class $(Y = 1, A = 1)$ has no support on a subset of the feature space that is occupied by the negative class $ (Y = 0, A = 1) $, which is necessary for identifiability. Under the assumption (\ref{eq:invariance}), the primary difficulty arises from the fact that the component $ p_{11}(\x)$, corresponding to the subgroup $(Y = 1, A = 1)$, is not directly observable in either the source or target domain.
	
	To elucidate this observation, we denote $\p_0(\x) = \{ p_{10}(\x), p_{00}(\x) \}\trans$, and then the observed data log-likelihood of one generic observation in our UDA setting is proportional to:
	\bse
	&&I_{110} \log p_{10}(\x) 
	+ I_{101} \log p_{01}(\x) 
	+ I_{100} \log p_{00}(\x) \\
	&& + I_{0\cdot1} \log \left\{ \beta_{11} p_{11}(\x) + (1 - \beta_{11} - \bb^\top \mathbf{1}) p_{01}(\x) \right\}
	+ I_{0\cdot0} \log \left\{ \bb^\top \p_0(\x) \right\}.
	\ese
	In this formulation, the parameter with finite dimension is $\bb$. 
	The model involves four nonparametric nuisance components: $p_{11}(\x)$, $p_{10}(\x)$, $p_{01}(\x)$, and $p_{00}(\x)$. 
		\begin{lemma}\label{lem:iden}
		Assume $\beta_{11} = 0$ and $p_{10}(\x) \neq p_{00}(\x)$, then all components except $ p_{11}(\x)$ are identifiable.
		Assume $0<\beta_{11} <1$ and is known, and $p_{10}(\x) \neq p_{00}(\x)$, then all components in the model are identifiable.  
	\end{lemma}
	The proof of Lemma~\ref{lem:iden} is provided in Appendix~\ref{sec:supp:method}.
	The identification conditions in Lemma~\ref{lem:iden} are intuitive and reasonable.
	If $\beta_{11}=0$, it degenerates to the
	situation that the source and target domains have the same support on both label $Y$ and background $A$, then the component $p_{11}(\x)$ is no longer
	relevant. 
	Also, if $p_{10}(\x) = p_{00}(\x)$, the subpopulations
	of $(Y=1, A=0)$ and $(Y=0, A=0)$ become indistinguishable, and hence the individual probabilities $\beta_{10}$ and $\beta_{00}$ are not separately identifiable. 
	Overall, these conditions are natural to ensure the problem is well-posed.

	\subsection{Estimating parameters of interest}\label{sec:estbeta}	

	To estimate the parameter $\bb$, we consider the distribution of attributes $\x$ in the subpopulation defined by $(R=0, A=0)$. 
	By the law of total probability, we have
	\be
	p(\x\mid R=0,A=0)\pr(R=0,A=0)
	= p_{10}(\x)\, \beta_{10} (1 - \pi) + p_{00}(\x)\, \beta_{00} (1 - \pi),\label{eq:distmatch}
	\ee
    subject to the constraint
	\be\label{eq:constraint}
	\pr(R=0, A=0) = \beta_{10}(1 - \pi) + \beta_{00}(1 - \pi).
	\ee
	Note that the distribution $p(\x \mid R=0, A=0)$ is identifiable from the target population. 
	The distributions $p_{10}(\x)$ and $p_{00}(\x)$ can be consistently estimated from the source population subgroups $(R=1, Y=1, A=0)$ and $(R=1, Y=0, A=0)$, respectively.
	Thus, the parameters $\bb = (\beta_{00}, \beta_{10})\trans$ can be estimated by minimizing a suitable discrepancy measure between the two sides of~(\ref{eq:distmatch}), such as an $L_2$ norm or a divergence-based criterion (e.g., Kullback--Leibler divergence), subject to the constraint in~(\ref{eq:constraint}).
    Therefore, we reformulate the estimation of \(\bb\) as a constrained distribution matching problem:
\begin{align}\label{eq:min}
    \wh\bb=\argmin_{\bb}D\left\{ \wh{p}(\x\mid R=0,A=0) \middle\| \{\wh{p}_{10}(\x)\beta_{10}+\wh{p}_{00}(\x)\beta_{00}\}/\wh\pr(A=0|R=0)\right\},
\end{align}
subject to $\wh\pr(A=0 |R=0) = \beta_{10} + \beta_{00}$, where
	$D$ denotes a discrepancy measure between probability distributions over the covariate space $\mX$. Among various choices for $D$, we adopt the Kullback–Leibler (KL) divergence due to its favorable analytical and computational properties.
	To facilitate optimization, we relax the constraint in~(\ref{eq:min}) and reformulate the objective under KL divergence, as summarized in the following lemma.
		\begin{lemma}\label{lem:max}
		Let $D$ be the Kullback–Leibler divergence. Then the solution $\wh\beta_{10}$ to the minimization problem (\ref{eq:min}) is given by
		\be\label{est:beta}
		\arg\max_{\beta_{10}}
\wh{E}\left(\log [\wh\xi_0(\X)  \wh{b}_1^{-1}\beta_{10}+\{1-\wh\xi_0(\X)\}  (1-\wh{b}_1)^{-1}(\wh\varrho-\beta_{10})]\middle|R=0,A=0\right),
		\ee
		where, for simplicity, $b_1=\pr(Y=1|R=1,A=0)$, $\varrho=\pr(A=0|R=0)$ and 
        $\wh{E}$ represents the empirical average.
	\end{lemma} 
The proof of Lemma \ref{lem:max} is provided in Appendix \ref{sec:supp:method}. A key advantage of minimizing the KL divergence is that it circumvents the need to explicitly estimate the generative probabilities $p_{10}(\x)$ and $p_{00}(\x)$, which are often difficult to model accurately in high dimensions. Instead, it suffices to estimate one background-specific prediction probability  $\xi_0(\x)$ using standard classification techniques on the source domain restricted to $A=0$.

Finally, based on all of the above discussions, we summarize the implementation details of our proposed method in Algorithm~\ref{alg}.

	\begin{algorithm}[htbp]
		\caption{Implementation details of our proposed method.}\label{alg}
		\label{alg:estimator}
		\renewcommand{\algorithmicrequire}{\textbf{Input:}}
		\renewcommand{\algorithmicensure}{\textbf{Output:}}
		\begin{algorithmic}[1]
			\Require Observed  source domain data $\{(\X_i, Y_i, A_i, R_i=1)\}_{i=1}^{n_1}$ and target domain data $\{(\X_i, A_i, R_i=0)\}_{i=1}^{n_0}$.
			\Ensure Estimated benchmark predictive probabilities $\wh\xi(\x)$ and $\wh\xi_0(\x)$, proposed predictive probabilities for the target $\wh\eta(\x)$, $\wh\eta_1(\x)$ and $\wh\eta_0(\x)$; and subpopulation proportions $\wh\alpha_{ya}$, $\wh\beta_{ya}$.
			\State Estimate $\xi(\x)$ (defined in (\ref{eq:xiandxi0})) using data $\{(\X_i,Y_i,R_i=1):i=1,\cdots, n_1\}$, as $\wh\xi(\x)$;
			\State Estimate $\xi_0(\x)$ (defined in (\ref{eq:xiandxi0})) using data $\{(\X_i,Y_i, A_i=0, R_i=1):i=1,\cdots, n_1\}$, as $\wh\xi_0(\x)$;
			\State Estimate $\tau_r(\x)$ (defined in Proposition \ref{pro:relation}) using data $\{(\X_i,A_i,R_i=r):i=1,\cdots n_r\}, r=0,1$, as $\wh\tau_r(\x)$;
			\State Estimate $\kappa(\x)$ (defined in (\ref{eq:kappa})) using data $\{(\X_i,R_i,A_i=1):i=1,\cdots n\}$, as $\wh\kappa(\x)$;
			\State Estimate $\bb$ and $\alpha_{y,a}$ following (\ref{est:beta}) and (\ref{est:alpha}), as $\wh\bb$ and $\wh\alpha_{ya}$ for $(y,a)\in \{0,1\}$;
        \State Estimate $\eta_1(\x)$, $\eta_0(\x)$ and $\eta(\x)$ following (\ref{eq:objection}), as $\wh\eta_1(\x)$, $\wh\eta_0(\x)$ and $\wh\eta(\x)$.
		\end{algorithmic}
	\end{algorithm}

The above method adopts the idea of distribution matching. Alternatively, one may consider matching only certain moments rather than the full distribution. Due to space constraints, we defer the details to Appendix \ref{method:moment}.

	\section{Theoretical Results}\label{sec:theory}

For the interest of space, we only present the results for the background-specific predictive probability with $A=0$.
The results for the other two predictive probabilities are parallel and can be similarly developed.
	To facilitate the analysis, we begin by formally defining the population-level (expected) objective function:
	\bse
	\bL(\xi_0,b_1,\beta_{10},\varrho)=E\left(\log [\xi_0(\X)  b_1^{-1}\beta_{10}+\{1-\xi_0(\X)\}  (1-b_1)^{-1}(\varrho-\beta_{10})]\middle|R=0,A=0\right),
	\ese
	with its empirical version $\wh{\bL}(\xi_0,b_1,\beta_{10},\varrho)$.
	
	\begin{assumption}\label{ass1}
    Define $\f(\x) = \{ f_0(\x), f_1(\x) \}\trans$, where
		$f_0(\x)=\log\{\xi_0(\x)\}-\frac{1}{2}[\log\{\xi_0(\x)\}+\log\{1-\xi_0(\x)\}]$ and $f_1(\x)=\log\{1-\xi_0(\x)\}-\frac{1}{2}[\log\{\xi_0(\x)\}+\log\{1-\xi_0(\x)\}]$, and the corresponding estimate is $\{\wh{f}_k(\x)\}_{k=0}^1$.
		There exist a constant $c > 0$ and a sequence  $r_{n_{1\cdot 0}} \to 0$ such that, for almost every $\x$, we have
		\bse
		\pr\left( \| \widehat{\f}(\x) - \f(\x) \|_2 > t \right) \leq \exp\left\{-t^2/(c^2 r_{n_{1\cdot 0}}^2)\right\}, \quad \forall\ t > 0.
		\ese
	\end{assumption}
    \begin{remark}
      Note that the tail bound described in Assumption \ref{ass1} is intended to hold uniformly for every $n_{1\cdot 0}$ when estimating $\wh{f}_k$ for $k=0,1$. In other words, for each subsample size $n_{1\cdot 0}$, we have a $r_{n_{1\dot 0}}$ such that the corresponding estimators $\wh{f}_k$ for $k=0,1$ are required to satisfy the stated concentration inequality. This inequality is analogous to Hoeffding's inequality and provides a non-asymptotic concentration bound on the estimation error. Similar assumptions have also been adopted in recent work (e.g., \cite{maity2022understanding}, \cite{tsybakov2007fast}).
    \end{remark}

	
	\begin{theorem}\label{thm:conver}
		Suppose Assumption~\ref{ass1} holds. Define
		$\chi_n = r_{n_{1\cdot 0}} \sqrt{\log(n_{0\cdot 0})} + n_{1\cdot 0}^{-1/2} + n_{0\cdot 0}^{-1/2}$. 
		Then, there exists a constant $c_{10} > 0$ such that for any $\delta > 0$, with probability at least $1 - 6\delta$, we have
		\bse
		\| \widehat{\boldsymbol{\beta}} - \boldsymbol{\beta} \|_1 \leq c_{10} \chi_n \sqrt{\log(1/\delta)}.
		\ese
	\end{theorem}
	
The proof of Theorem \ref{thm:conver} is provided in Appendix \ref{sec:supp:theory}.
Theorem~\ref{thm:conver} establishes the consistency of the estimator $\wh{\bb}$, provided that $r_{n_{1\cdot 0}} \sqrt{\log(n_{0\cdot 0})} \to 0$ as $n_{1\cdot 0}, n_{0\cdot 0}\to \infty$. 

Theorem \ref{thm:conver} establishes the upper bound of the parameter estimate. Beyond parameter estimation, an important question is how the resulting estimator performs in downstream prediction tasks. 

With any loss function $\ell(\cdot)$ and some function $h(\cdot)$,  for the background-specific prediction model with $A=0$, the conditional risk is
	\be\label{eq:relation0}
	E[\ell\{h(\X), Y\} \mid R = 0, A = 0]= E[\ell\{h(\X), Y\} \, w(Y) \mid R = 1, A = 0],
	\ee
	where, for simplicity, we write $w(y) = \frac{\pr(y \mid R = 0, A = 0)}{\pr(y \mid R = 1, A = 0)}$.
	One can derive that $w(1)=\frac{\beta_{10}(\alpha_{00}+\alpha_{10})}{\alpha_{10}(\beta_{00}+\beta_{10})}$ and $w(0)=\frac{\beta_{00}(\alpha_{00}+\alpha_{10})}{\alpha_{00}(\beta_{00}+\beta_{10})}$.
	To evaluate the performance of the predictive probability, it can be approximated as
	$\wh{E}[\ell\{h(\X), Y\} \, \wh{w}(Y) \mid R = 1, A = 0]$.
	Furthermore, the model can be fine-tuned specifically for the target subgroup by minimizing the reweighted empirical risk:
	\be\label{REM}
	\wh{h}_{\wh{w}} \in \arg\min_{h \in \mF} \wh{E}[\ell\{h(\X), Y\} \, \wh{w}(Y)\mid R = 1, A = 0],
	\ee
	where $\mF$ is a suitable function class. 
	For the other two predictive probabilities, the relations analogous to (\ref{eq:relation0}) can be derived similarly.
	For the background-specific predictive probability  with $A=1$, one can derive  $E[\ell\{h(\X),Y\}|R=0,A=1]$ as
	\bse
	E[\ell\{h(\X),Y=1\}|R=0,A=1] 
	-E\left(\left[\ell\{h(\X),Y=1\}-\ell\{h(\X),Y=0\}\right]|Y=0,A=1\right)\frac{\beta_{01}}{\beta_{01}+\beta_{11}}.
	\ese
	For the overall predictive probability, the conditional risk $E[\ell\{h(\X),Y\}|R=0]$ is
	\bse
	&& E[\ell\{h(\X), Y\} \, w(Y) \mid R = 1, A = 0](\beta_{10} +\beta_{00})
	+E[\ell\{h(\X),Y=1\}|R=0,A=1](\beta_{01}+\beta_{11})\\
	&&-E\left(\left[\ell\{h(\X),Y=1\}-\ell\{h(\X),Y=0\}\right]|Y=0,A=1\right)\beta_{01}.
	\ese

    The above relations show that the target-domain risk can be expressed in terms of appropriately reweighted risks computed from the observed source data. This motivates the use of weighted empirical risk minimization for learning classifiers tailored to the target domain. 
    
    We now establish a generalization bound for the fitted model~(\ref{REM}), which is obtained via weighted empirical risk minimization over the source subgroup. Let $\mF$ denote the hypothesis class of classifiers. For any $h \in \mF$ and a weight function $w(y): y \rightarrow \R$, we define the population-level weighted loss and its empirical counterpart based on the source subgroup data as follows:
	\bse
	\mL_1(h,w)&=& E\left[\ell\{h(\X),Y\}w(Y)\big|A=0,R=1\right],\\
	\wh\mL_1(h,w)&=&\wh{E}\left[\ell\{h(\X),Y\}w(Y)\big|A=0,R=1\right].
	\ese
	We also define the population loss on the target subgroup as: $\mL_0(h)= E\left[\ell\{h(\X),Y\}\big|R=0,A=0\right]$. Clearly, $\mL_1(h,w)=\mL_0(h)$.

	To establish our generalization bound, we utilize the concept of Rademacher complexity \citep{bartlett2002rademacher}, denoted as $\mR_n(\mG)$ (see Appendix~\ref{sec:supp:theory} for details), and impose the following assumption on the loss function:
	\begin{assumption}\label{ass2}
		The loss function $\ell$ is uniformly bounded; that is, there exists a constant $B>0$ such that
		\bse
		|\ell\{h(\x),y\}|\leq B\ \text{for any}\ h\in \mF, \x\in \mX\subset \R^q, \text{and}\ y\in \{0,1\}.
		\ese
	\end{assumption}
	We now present the generalization bound for the learned model, with its proof provided in Appendix~\ref{sec:supp:theory}.
	\begin{proposition}\label{pre}
		Under Assumptions \ref{ass1} and \ref{ass2}, let $\wh h_{\wh w}=\argmin_{h\in \mF}\wh\mL_1(h,\wh{w})$ be the classifier obtained by minimizing the reweighted empirical risk on the source subgroup. Then, there exist constants $c,d>0$ such that, with probability at least $1-7\delta$, the following generalization bound holds:
		\bse
		\mL_0(\wh h_{\wh w})-\min_{h\in\mF}\mL_0(h)\leq 2 \mR_{n_{1\cdot 0}}(\mG)+d B\|\wh\bb- \bb\|_1+c\left\{\sqrt{\frac{\log(1/\delta)}{n_{1\cdot 0}}}+\sqrt{\frac{\log(1/\delta)}{n_{0}}}\right\},
		\ese
		where $\mG=\{w(y)\ell\{h(\x),y\}:h\in\mF\}$, and $\mR_{n_{1\cdot 0}}(\mG)$ denotes its Rademacher complexity as defined in Appendix~\ref{sec:supp:theory}.
	\end{proposition}
		\begin{remark}
		Proposition~\ref{pre} indicates that the generalization bound depends on the estimation error $\|\wh\bb-\bb\|_1$, which can be directly controlled based on the conditions listed in Assumption \ref{ass1}, implying that different estimation procedures for $\bb$ will yield different upper bounds. In Theorem~\ref{thm:conver}, we established an upper bound for the estimation error of $\wh{\bb}$, which directly leads to a refined generalization bound for the learned classifier $\wh{h}_{\wh{w}}$. Specifically, for any $\delta>0$, with probability at least $1-13\delta$, the following inequality holds:
		\bse
		\mL_0(\wh h_{\wh w})-\min_{h\in\mF}\mL_0(h)\leq 2 \mR_{n_{1\cdot 0}}(\mG)+d Bc_{10}\chi_n\sqrt{\log(1/\delta)}+c\left\{\sqrt{\frac{\log(1/\delta)}{n_{1\cdot 0}}}+\sqrt{\frac{\log(1/\delta)}{n_{0}}}\right\},
		\ese
		where $c_{10}$ is the constant appearing in Theorem~\ref{thm:conver}, and $\alpha_n$ characterizes the convergence rate of $\wh{\bb}$.
	\end{remark}

\section{Synthetic Data Results}\label{add:synthetic}
	
We consider a structured data-generating process in which the covariate $\X \in \R^4$ is drawn from a distribution conditioned on a pair $(Y, A)$, where $Y \in \{0, 1\}$ denotes the class label and $A \in \{0, 1\}$ denotes the background. The generation begins by sampling $(Y, A)$ according to a predefined distribution.
	
In the \emph{source} domain, we consider $(Y, A) \in \{(0,0), (0,1), (1,0)\}$, each occurring with probability $1/3$. The covariate $\X \in \R^4$ is generated as $\X \sim N(\bmu_{\rm{YA}}, \I_4)$, where $\bmu_{\rm{YA}}$ denotes the mean vector for each combination and $\I_4$ is the $4 \times 4$ identity matrix. The stratum $(1,1)$ is excluded from the source. In the \emph{target} domain, all four combinations $(Y, A) \in \{0,1\}^2$ appear with equal probability $1/4$, and $\X$ is drawn from the same distribution $N(\bmu_{\rm{YA}}, \I_4)$ with distinct means:
\bse
\bmu_{00} = (1, 0, 0, 0)\trans, \quad \bmu_{01} = (0, 0, 1, 0)\trans, \quad \bmu_{10} = (0, 1, 0, 0)\trans, \quad \bmu_{11} = (0, 0, 0, 1)\trans.
\ese
	
We follow Algorithm~\ref{alg} to compute the conditional probabilities required by the proposed approach as well as the two naive benchmarks. 
Specifically, we calculate the five key conditional probabilities needed for implementation: $\xi_0(\x)$, $\xi(\x)$, $\tau_0(\x)$, $\tau_1(\x)$, and $\kappa(\x)$, which, together with the estimators of the parameters $\{\beta_{ya}:y=0,1;a=0,1\}$ and $\{\alpha_{ya}:y=0,1;a=0,1\}$, determine the predictive probabilities $\eta_0(\x)$, $\eta_1(\x)$ and $\eta(\x)$ for the proposed method.
In addition, we compute the naive benchmark $\gamma(\x)$ following (\ref{eq:naive2}), 
which corresponds to the standard UDA method relying solely on the label shift assumption without accounting for the subgroup information. 

To assess the performance of the proposed approach as well as to compare with the two naive benchmarks, we conduct 100 simulations for each configuration and summarize the results using boxplots that compare ${\wh\eta(\x), \wh\xi(\x), \wh\gamma(\x)}$ across varying sample sizes. The left panel of Figure~\ref{fig:simulation_combined} displays the performance of $\wh\eta(\x), \wh\xi(\x)$ and  $\wh\gamma(\x)$ for $n_0 = 1000$ and $6000$, with $n_1$ ranging from $1000$ to $8000$, while the right panel shows the corresponding results for $n_1 = 1000$ and $6000$, with $n_0$ ranging from $1000$ to $8000$. Performance is evaluated using two standard metrics: accuracy and F$_1$ score. In both metrics, the proposed approach consistently outperforms the benchmark estimators. 
Notably, the accuracy of $\wh\eta(\x)$ steadily improves with larger $n_0$ (or $n_1$), further demonstrating the robustness and effectiveness of the proposed method.

\begin{figure}[htbp]
    \centering
    \begin{subfigure}[b]{0.48\textwidth}
        \centering
        \includegraphics[width=\textwidth]{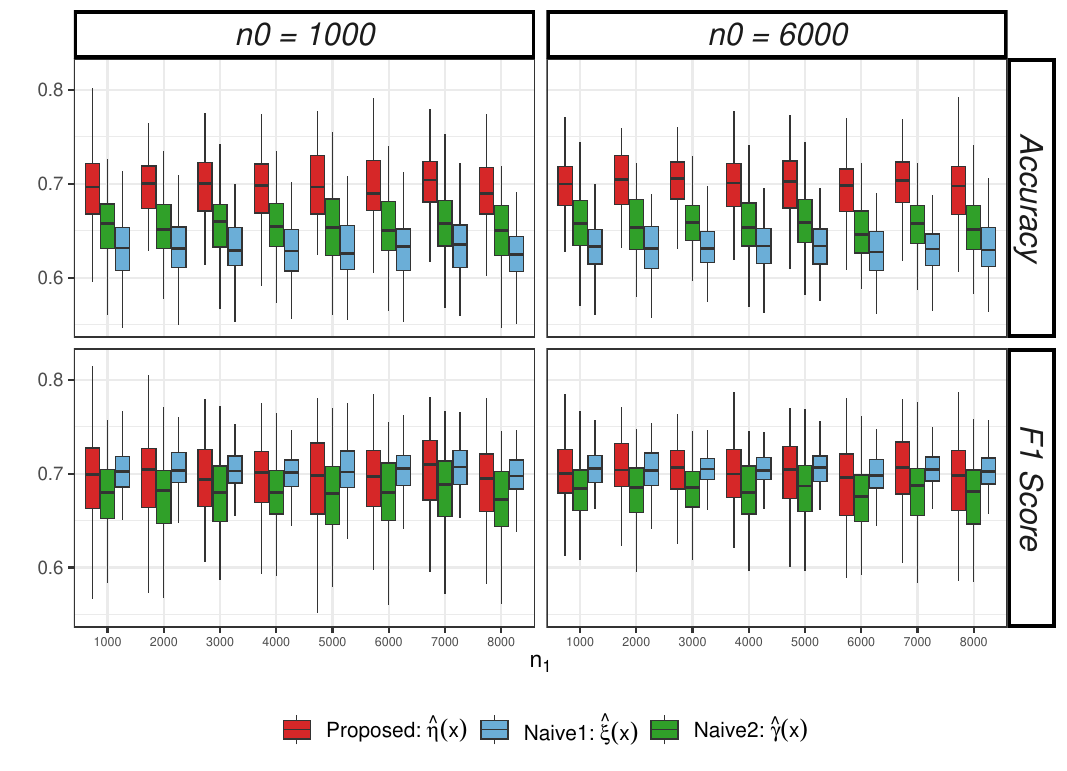}
    \end{subfigure}
    \hfill
    \begin{subfigure}[b]{0.48\textwidth}
        \centering
        \includegraphics[width=\textwidth]{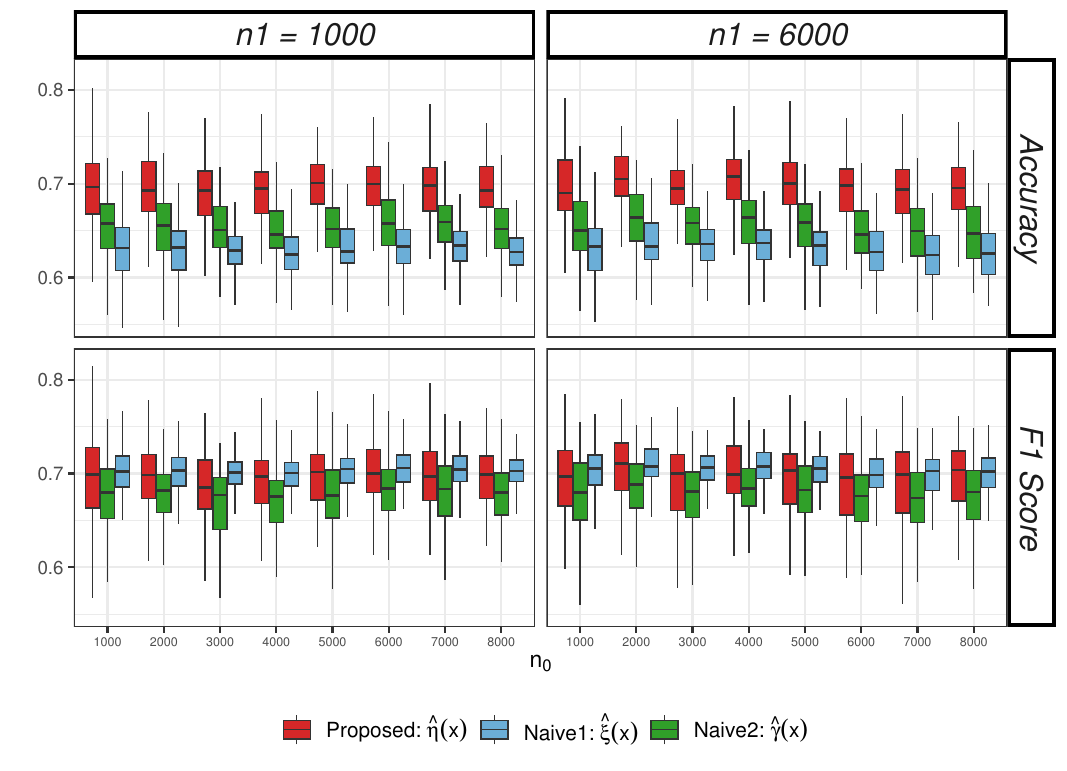}
    \end{subfigure}
\caption{Left panel: Accuracy and F$_1$ score of $\wh\eta(\x)$, $\wh\xi(\x)$, and $\wh\gamma(\x)$ with fixed $n_0 = 1000/6000$ but $n_1$ varies. 
Right panel: Accuracy and F$_1$ score of $\wh\eta(\x)$, $\wh\xi(\x)$, and $\wh\gamma(\x)$ with fixed $n_1 = 1000/6000$ but $n_0$ varies.}    \label{fig:simulation_combined}
\end{figure}

	\section{Experiments}\label{numeric}
	
In this section, we analyze the \textbf{Waterbirds} dataset \citep{sagawa2019distributionally}, which consists of 11,788 images. 
This public dataset is widely used to investigate spurious correlations in image classification. 
It is also well aligned with our problem setting of unsupervised domain adaptation under \emph{structured missingness}, where specific combinations of labels and backgrounds are systematically absent in the labeled source domain, while labels are entirely unobserved in the target domain.
    The label $Y=1$ denotes a waterbird and $Y=0$ landbird.
    The background $A=1$ corresponds to a water background and $A=0$ a land background. It yields four label–background subpopulations, as summarized in Table~\ref{tab:waterbirds_counts}.

    	\begin{table}[htbp]
		\centering
		\caption{Empirical joint distribution of $(Y,A)$ in the Waterbirds dataset, with varied values of $a$, $b$ and $c$, $0<a,b,c<1$.}
		\label{tab:waterbirds_counts}
		\resizebox{\textwidth}{!}{
			\begin{tabular}{cccccccc}
				\toprule
				$Y$ & $A$ & Description & Count & Total Proportion & Proportion in Source & Proportion in Target \\
				\midrule
                    1 & 1 & Waterbird on water & 1832 & 0.155 & 0       &    0.155\\
                    0 & 1 & Landbird on water & 2905  & 0.246 & $0.246a$ & $0.246(1-a)$\\
                    1 & 0 & Waterbird on land & 831  & 0.071 & $0.071b$ & $0.071(1-b)$\\
                    0 & 0 & Landbird on land  & 6220 & 0.528 & $0.528c$ & $0.528(1-c)$\\\bottomrule
		\end{tabular}}
	\end{table}
	
	To construct a structured domain adaptation problem, we partition the full dataset into a \emph{source domain} ($R=1$) and a \emph{target domain} ($R=0$). Specifically, we allocate samples from three subgroups, $(Y=0,A=1)$, $(Y=1,A=0)$ and $(Y=0,A=0)$, into the source domain, with allocation rates denoted by parameters $a$, $b$, and $c$, respectively. The remaining subgroup, $(Y=1,A=1)$, is deliberately \emph{excluded} from the source domain and appears only in the target domain. This setting reflects real-world scenarios in which a specific combination of label and background is structurally missing from labeled datasets due to systematic data collection biases or constraints. In the target domain, all four subgroups are retained, but the label variable $Y$ is treated as unobserved. 

    To implement the proposed method, we apply the distribution matching approach to estimate the subclass proportions in the target domain. 
	For feature extraction, we embed each image into a 512-dimensional feature vector using a ResNet-18 model \citep{he2016deep} and a ViT-16 model \citep{heo2021rethinking}, both pre-trained on ImageNet \citep{deng2009imagenet}, without additional fine-tuning. These embeddings serve as covariate $\X\in \R^{512}$  in our downstream analysis. Based on these feature vectors, we fit logistic regression models with $L_2$-regularization to estimate five key conditional probabilities required by both our proposed method and benchmark procedures: $\xi_0(\x)$, $\xi(\x)$, $\tau_0(\x)$, $\tau_1(\x)$ and $\kappa (\x)$.

	\begin{figure}[htbp]
		\centering
		\includegraphics[width=\textwidth]{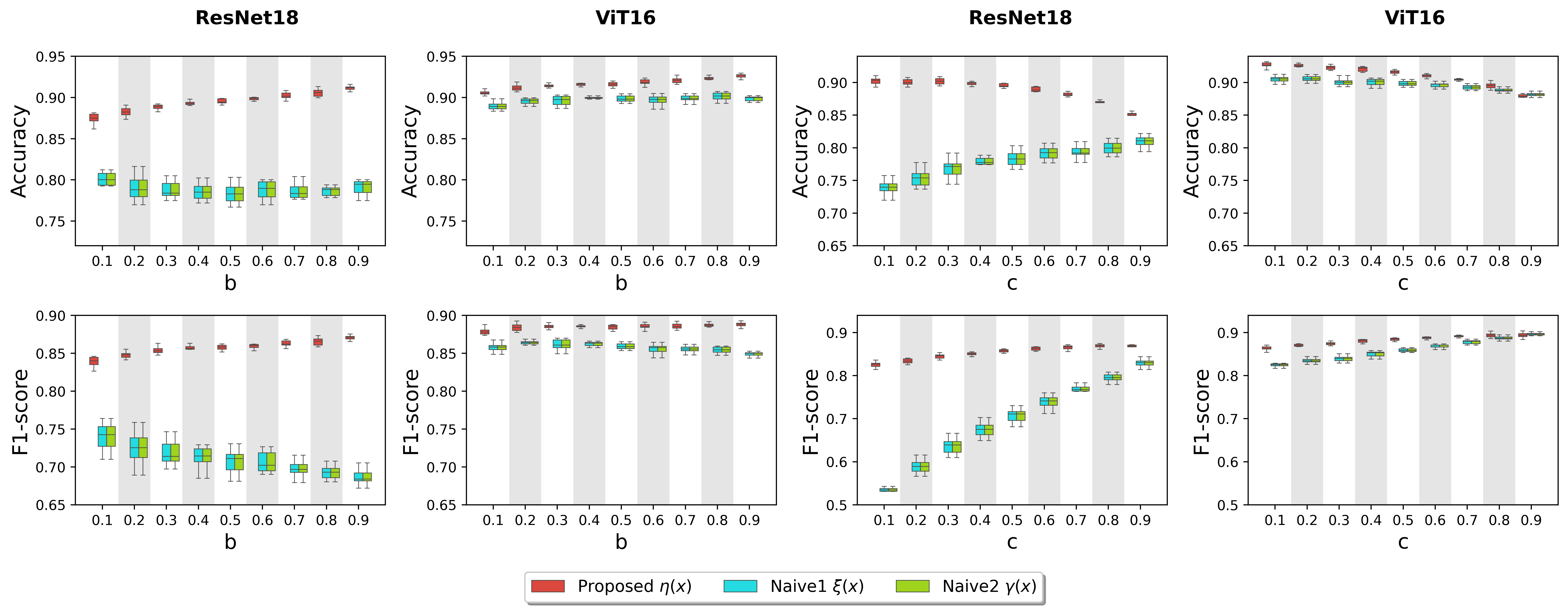}  
		\caption{Performance comparison of our proposed estimator $\eta(\x)$, and the naive methods $\xi(\x)$ and $\gamma(\x)$ under the setting \underline{$a = 0.7$} with either $c=0.5$ and varying $b$ or $b=0.5$ and varying $c$.}
		\label{fig:plot07}
	\end{figure}
	
		\begin{figure}[htbp]
		\centering
		\includegraphics[width=\textwidth]{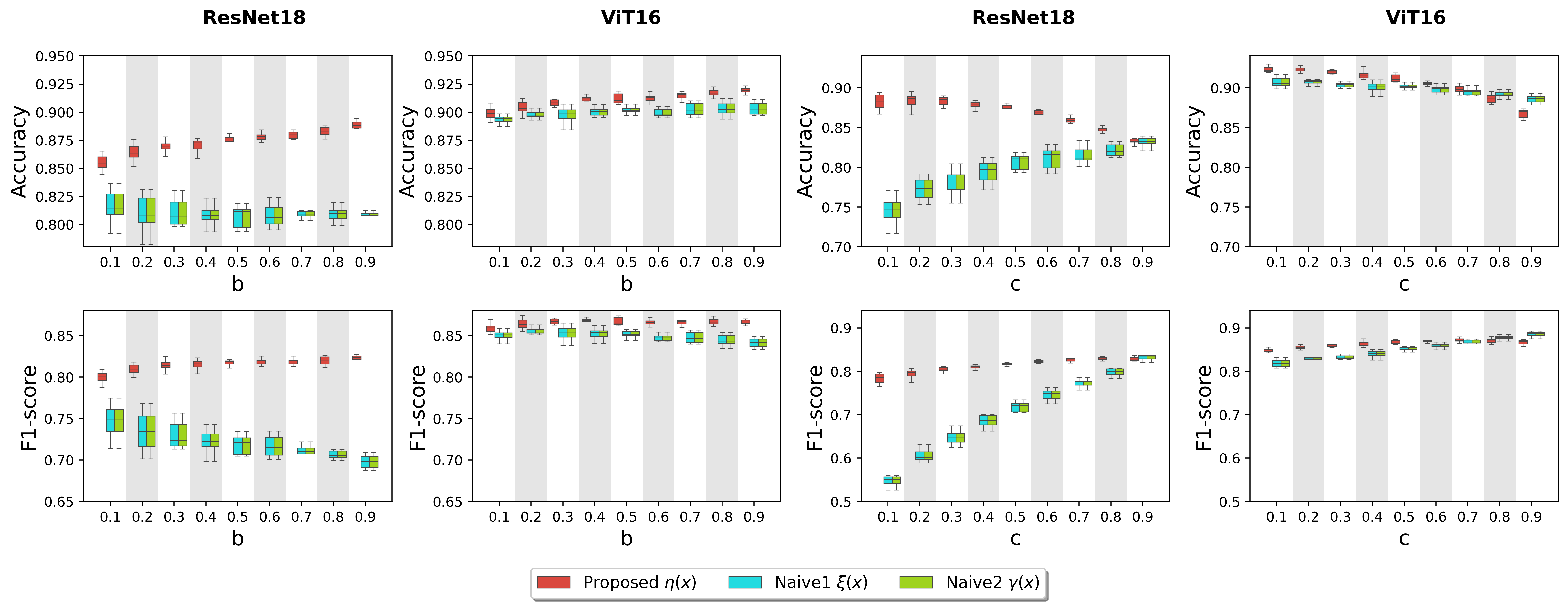} 
		\caption{Performance comparison of our proposed estimator $\eta(\x)$, and the naive methods $\xi(\x)$ and $\gamma(\x)$ under the setting \underline{$a = 0.5$} with either $c=0.5$ and varying $b$ or $b=0.5$ and varying $c$.}
		\label{fig:plot05}
	\end{figure}
	
		\begin{figure}[htbp]
		\centering
		\includegraphics[width=\textwidth]{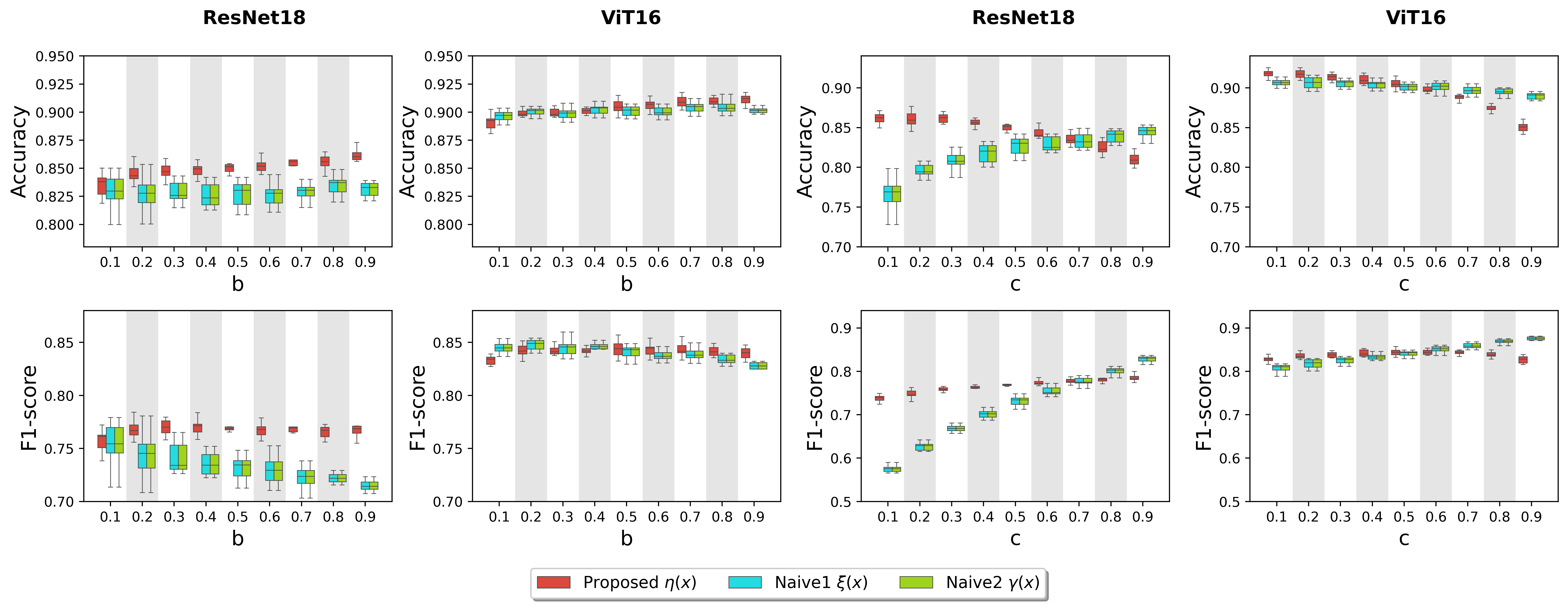}  
		\caption{Performance comparison of our proposed estimator $\eta(\x)$, and the naive methods $\xi(\x)$ and $\gamma(\x)$ under the setting \underline{$a = 0.3$} with either $c=0.5$ and varying $b$ or $b=0.5$ and varying $c$.}
		\label{fig:plot03}
	\end{figure}

	For empirical evaluation, we fix the subclass sampling rate at $a = 0.7$ in the source domain and systematically vary the remaining subclass inclusion rates by setting either $b = 0.5$ with $c \in \{0.1, 0.2, \ldots, 0.9\}$, or $c = 0.5$ with $b \in \{0.1, 0.2, \ldots, 0.9\}$. For each configuration, the data generation process is repeated 50 times to account for sampling variability. We assess performance using two widely adopted classification metrics: accuracy and F$_1$ score.
	Figure~\ref{fig:plot07} presents boxplots summarizing the distributions of the metrics across repeated experiments for our proposed estimator $\wh\eta(\x)$ and the two naive estimators $\wh\xi(\x)$ and $\wh\gamma(\x)$. 
	The left panel shows the results based on features extracted by ResNet-18, while the right panel shows the results based on features extracted by ViT-16.
	We also reduce the proportion $a$ from $0.7$ to $0.5$ and $0.3$, and similarly demonstrate the results in Figure~\ref{fig:plot05} and Figure~\ref{fig:plot03}, respectively.
	
	The observations and conclusions from this experiment can be summarized as follows.
	First, the proposed \emph{correct} predictive probability $\eta(\x)$ consistently achieves the best performance across most scenarios. This is particularly evident when the subpopulation $(Y=0,A=1)$ constitutes a larger proportion $(a=0.7)$ in the source domain (Figure~\ref{fig:plot07}). 
	This behavior is logical; because $(Y=0,A=1)$ is the only class containing information about $A=1$, a higher representation of this subgroup helps compensate for the absence of the $(Y=1,A=1)$ subgroup.
	Second, the performance of the proposed $\eta(\x)$ can degrade, even falling behind the two \emph{incorrect} naive benchmarks, if the proportion of either $(Y=1,A=0)$ or $(Y=0,A=0)$ becomes exceedingly small either in the source domain or in the target. 
	In such cases, the source domain may effectively suffer from multiple missing subgroups (beyond just $(Y=1,A=1)$), elevating the problem's difficulty well beyond the scope of this paper.
	Finally, choosing ViT-16 generally brings much better performance than using ResNet-18, especially for the two naive benchmarks.
	
\section{Discussions and Conclusions}

In this paper, we introduce a novel unsupervised domain adaptation setting where an entire label-background subpopulation is absent from the source domain, a scenario motivated by real-world data collection constraints. 
Despite this structured missingness, we show that accurate prediction in the target domain is still achievable. 
We develop a theoretical framework that enables such prediction by estimating subpopulation proportions in the target through distribution matching. 
We provide rigorous guarantees, including statistical consistency as well as upper bounds on the target-domain prediction error. Empirically, our method outperforms standard baselines that overlook structured missingness, especially in prediction performance for the unobserved subpopulation. 
Overall, our framework provides a rigorous characterization of model adaptation under subpopulation structured missingness, and enables robust domain adaptation in such a challenging scenario.

Our theoretical framework is built upon structured conditional invariance and mixture proportion estimation. These tools naturally generalize to multi-class labels for $n_y$ species and multi-level (or even continuous) environment variables for $n_a$ species. In fact, the identification strategy and distribution-matching estimation carry over to larger joint label-environment spaces, though at the cost of heavier notation and more complex optimization. Technically, at this general multi-label and multi-background situation, the model identification considerations (see discussion in Section \ref{sec:model}) becomes more complex. At this situation, one can identify both $\pr(\X, A=a|R=0)$ as well as $\pr(A=a|R=0)$, which in total $2n_a-1$ quantities, while one has in total $n_yn_a$ unknown quantities, including $\pr(Y=y, A=a|R=0)$ and the unobservable subpopulation distribution $\pr(\X|Y=1,A=1)$. To make sure this model is identifiable, one needs to make $(n_y-2)n_a+1$ anchor set assumptions. For example, when $n_y=3$ and $n_a=2$, 3 anchor set assumptions are needed. Interestingly, as long as the label is binary $n_y=2$, one anchor set assumption is sufficient if only one subpopulation is missing in the source. In the setting we consider in the paper, $n_y=n_a=2$, so we only need to make one anchor set assumption.

The invariance assumption imposed in (\ref{eq:invariance}) can be viewed as a conditional, or more nuanced, extension of the standard label shift assumption. 
While standard label shift assumes that the distribution of $\X$ has no change across domains conditional on the label $Y$, equation (\ref{eq:invariance}) postulates that its distribution remains the same conditional on both the label $Y$ and the environment $A$. As discussed in Section~\ref{sec:prelim}, this assumption is justifiable in a variety of practical contexts.
In practice, however, the environment $A$ itself may undergo shifts. 
For instance, the source domain might contain only two backgrounds (e.g., water and land), whereas the target domain introduces a third (e.g., sky). 
This scenario effectively transitions into an out-of-distribution (OOD) generalization problem, where the shift occurs with respect to the environment $A$ rather than the label $Y$. 
Addressing a missing subgroup in the source domain within this OOD context extends significantly beyond the scope of this paper, as the model identification challenges discussed in Section~\ref{sec:model} become substantially more intractable. 
Consequently, this problem remains an important area for future work.

\subsubsection*{Broader Impact Statement}
This work addresses a practical limitation of unsupervised domain adaptation by enabling reliable prediction when an entire subpopulation is absent from the source data, a situation that commonly arises in healthcare, ecological monitoring, and other domains where certain subgroups are systematically underrepresented due to data collection constraints. 
By explicitly modeling structured missingness, our approach can help reduce biased predictions on groups that conventional methods would otherwise misclassify, contributing to fairer and more robust machine learning systems.

\subsubsection*{Acknowledgments}
Sharon Li is supported in part by the AFOSR Young Investigator Program under award number FA9550-23-1-0184, National Science Foundation (NSF) under awards IIS-2237037 and IIS-2331669, Office of Naval Research under grant number N00014-23-1-2643, Schmidt Sciences Foundation, Open
Philanthropy, Alfred P. Sloan Fellowship, and gifts from Google and Amazon.
Jiwei Zhao is supported in part by NSF under awards DMS-1953526, DMS-2122074 and DMS-2310942, and National Institutes of Health under award R01DC021431.
Any opinions, findings, and conclusions or recommendations expressed in this paper are those of the author(s) and do not necessarily reflect the views of the funding agencies.

\bibliography{main}
\bibliographystyle{tmlr}

\newpage
\appendix
\section{Appendix}
	\subsection{Proofs and More Details in Section~\ref{sec:method}}\label{sec:supp:method}
	\begin{proof}[Proof of Proposition~\ref{pro:relation}]
		For $A=1$ case, note that
		\bse
		&& p(\x\mid R=0,A=1)\pr(R=0,A=1) \\
		&=& p(\x\mid R=0,Y=1,A=1)p_{011} + p(\x\mid R=0,Y=0,A=1)p_{001}.
		\ese
		Thus,
		\bse
		p_{11}(\x) = \frac{p(\x\mid R=0,A=1)\pr(R=0,A=1)-p_{01}(\x)p_{001}}{p_{011}}.
		\ese
		Then,
		\bse
		&& \pr(Y=1\mid \x,R=0,A=1) = \frac{p_{11}(\x)p_{011}}{
			p(\x,R=0,A=1)
		}\\
		&=& \frac{p(\x\mid R=0,A=1)\pr(R=0,A=1)-p_{01}(\x)p_{001}}{p(\x\mid R=0,A=1)\pr(R=0,A=1)} \\
		&=& \frac{p(\x\mid R=0,A=1)p_{0\cdot1}-p_{01}(\x)\beta_{01}(1-\pi)}{p(\x\mid R=0,A=1)p_{0\cdot1}}
		\ese
		Note that
		\bse
		\pr(R=1\mid \x,A=1)=\frac{p_{01}(\x)\alpha_{01}\pi}{p_{01}(\x)\alpha_{01}\pi+p(\x\mid R=0,A=1)p_{0\cdot1}}
		\ese
		gives
		\bse
		\frac{p_{01}(\x)}{
			p(\x\mid R=0,A=1) p_{0\cdot1}}
		=\frac{\pr(R=1\mid \x,A=1)}{\alpha_{01}\pi\{1-\pr(R=1\mid \x,A=1)\}}.
		\ese
		Hence,
		\bse
		\pr(Y=1\mid \x,R=0,A=1) = 1-\frac{\beta_{01}(1-\pi)}{\alpha_{01}\pi}\frac{\pr(R=1\mid \x,A=1)}{1-\pr(R=1\mid \x,A=1)}.
		\ese
		Note that
		\bse
		\pr(R=1\mid \x,A=1)=\frac{p(\x\mid R=1,A=1)\alpha_{01}\pi}{p(\x\mid R=1,A=1)\alpha_{01}\pi+p(\x\mid R=0,A=1)p_{0\cdot1}}
		\ese
		gives
		\bse
		\frac{p(\x\mid R=1,A=1)}{
			p(\x\mid R=0,A=1) p_{0\cdot1}}
		=\frac{\pr(R=1\mid \x,A=1)}{\alpha_{01}\pi\{1-\pr(R=1\mid \x,A=1)\}}.
		\ese
		Hence,
		\bse
		\pr(Y=1\mid \x,R=0,A=1) = 1-\frac{\beta_{01}(1-\pi)}{\alpha_{01}\pi}\frac{\pr(R=1\mid \x,A=1)}{1-\pr(R=1\mid \x,A=1)}\left\{\frac{p(\x|Y=0,A=1)}{p(\x|R=1,A=1)}\right\}.
		\ese

    For $A=0$ case, note that
    		\bse
		&&\pr(Y=1\mid \x,R=0,A=0) = \frac{\pr(Y=1,\x,R=0,A=0)}{\pr(Y=1,\x,R=0,A=0)+\pr(Y=0,\x,R=0,A=0)}\\
&=&\frac{\pr(\x, Y=1,R=1,A=0)\frac{\pr(Y=1,R=0,A=0)}{\pr(Y=1,R=1,A=0)}}{\pr(\x, Y=1,R=1,A=0)\frac{\pr(Y=1,R=0,A=0)}{\pr(Y=1,R=1,A=0)}+\pr(\x, Y=0,R=1,A=0)\frac{\pr(Y=0,R=0,A=0)}{\pr(Y=0,R=1,A=0)}}\\
&=&\frac{\frac{\beta_{10}}{\alpha_{10}} \xi_0(\x)}{\frac{\beta_{10}}{\alpha_{10}} \xi_0(\x)+\frac{\beta_{00}}{\alpha_{00}} \{1 - \xi_0(\x)\}}.
		\ese
By Bayes' rule, we obtain the following equation
\bse
\eta(\x) = \eta_1(\x) \tau_0(\x) + \eta_0(\x) \{1 - \tau_0(\x)\}.
\ese
	\end{proof}

	\begin{proof}[Proof of Lemma~\ref{lem:iden}]
		It is easy to see that, $\pi$, $\alpha_{10}$, $\alpha_{01}$, $p_{10}(\x)$, $p_{01}(\x)$ and $p_{00}(\x)$ are all identifiable.
		Now suppose that there are two different sets $p_{11}(\x)$, $\beta_{10}$, $\beta_{00}$ and $\wt p_{11}(\x)$, $\wt \beta_{10}$, $\wt \beta_{00}$ such that
		\be
		\beta_{11}p_{11}(\x)+(1-\beta_{11}-\beta_{10}-\beta_{00})p_{01}(\x) &=&
		\beta_{11}\wt p_{11}(\x)+(1-\beta_{11}-\wt \beta_{10}-\wt \beta_{00})p_{01}(\x), \n\\
		\beta_{10}p_{10}(\x)+\beta_{00}p_{00}(\x) &=& \wt \beta_{10}p_{10}(\x)+\wt \beta_{00}p_{00}(\x).\label{eq:2}
		\ee
		Now taking the integral with respect to $\x$ on both sides of the second equation above, it is clear that
		\bse
		\beta_{10}+\beta_{00}=\wt \beta_{10}+\wt \beta_{00}.
		\ese
		Plugging in back to the first equation above, we obtain
		\bse
		\beta_{11} \{p_{11}(\x)-\wt p_{11}(\x)\} = 0.
		\ese
		Since $\beta_{11}>0$, we obtain $p_{11}(\x)=\wt p_{11}(\x)$.
		Finally, (\ref{eq:2}) leads to
		$(\beta_{10}-\wt\beta_{10})p_{10}(\x)=
		(\wt \beta_{00}-\beta_{00})p_{00}(\x)$, which can only hold if
		$\beta_{10}=\wt\beta_{10}$ and $\wt \beta_{00}=\beta_{00}$ since
		$p_{10}(\x)\ne p_{00}(\x)$.
		This completes the proof.
	\end{proof}

	\begin{proof}[Proof of Lemma~\ref{lem:max}]
	\bse
	&&D\left\{p(\x|R=0,A=0)\middle\|\sum_{k=0}^1 p(\x|Y=k,A=0)\beta_{k0}\frac{1-\pi}{\pr(R=0,A=0)}\right\}\\
	&=&\int p(\x|R=0,A=0)\log \frac{p(\x|R=0,A=0)}{\sum_{k=0}^1 p(\x|Y=k,A=0)\beta_{k0}\frac{1-\pi}{\pr(R=0,A=0)}} \mathrm{d}\x\\
	&=&\int p(\x|R=0,A=0)\log \frac{p(\x|R=0,A=0)}{p(\x|R=1,A=0)} \mathrm{d}\x\\
	&&-\int p(\x|R=0,A=0)\log \frac{\sum_{k=0}^1 p(\x|Y=k,A=0)\beta_{k0}\frac{1-\pi}{\pr(R=0,A=0)}}{p(\x|R=1,A=0)} \mathrm{d}\x\\
	&=&\int p(\x|R=0,A=0)\log \frac{p(\x|R=0,A=0)}{p(\x|R=1,A=0)} \mathrm{d}\x\\
	&&-\int p(\x|R=0,A=0)\log \sum_{k=0}^1\frac{\pr(Y=k|\x,R=1,A=0)\beta_{k0}(1-\pi)\pr(R=1,A=0)}{\pr(R=1,Y=k,A=0)\pr(R=0,A=0)} \mathrm{d}\x.
	\ese
	
	Minimizing the above equation is equivalent to maximizing
	\bse
	\argmax_{\bb}E\left\{\log \sum_{k=0}^1\pr(Y=k|\x, R=1,A=0)\frac{\beta_{k0}}{\pr(Y=k|R=1,A=0)}\middle|R=0,A=0\right\},
	\ese
	subject to $\pr(R=0,A=0)=\beta_{10}(1-\pi)+\beta_{00}(1-\pi)$.
	
	We enforce this restriction as a constraint in the distribution matching problem:
	where $D$ is a discrepancy between probability distributions on $\mathcal{X}$.


	Define
	\bse
	\bL(\xi_0,b_1,\beta_{10},\varrho)=E\left(\log [\xi_0(\X)  b_1^{-1}\beta_{10}+\{1-\xi_0(\X)\}  (1-b_1)^{-1}(\varrho-\beta_{10})]\middle|R=0,A=0\right).
	\ese
	Its empirical version is 
	
	\bse
	\wh{\bL}(\xi_0,b_1,\beta_{10},\varrho)=\wh{E}\left(\log [\xi_0(\X)  b_1^{-1}\beta_{10}+\{1-\xi_0(\X)\}  (1-b_1)^{-1}(\varrho-\beta_{10})]\middle|R=0,A=0\right).
	\ese
	\end{proof}

	       \subsubsection*{An Alternative Approach for Estimating $\bb$}\label{method:moment}
	In the main text, we explore the use of distribution matching for estimating $\bb$. Alternatively, it is sufficient to only consider some moments instead of the whole 
    distribution. For any measurable function $\m(\x)$, the law of total expectation yields the identity:
	\be
	&&E\{\m(\x)\mid R=0,A=0\}\pr(R=0,A=0)\n\\
	&=&E\{\m(\x)\mid 1,0\}\beta_{10}(1-\pi) + E\{\m(\x)\mid 0,0\}\beta_{00}(1-\pi). \label{eq:momentmatch}
	\ee
	Rewriting equation~(\ref{eq:momentmatch}), we obtain the following linear system:
	\bse
	(1-\pi)p_{0\cdot0}^{-1} \left[ E\{\m(\x)\mid 1,0\}, E\{\m(\x)\mid 0,0\}\right]\bb 
	= E\{\m(\x)\mid R=0,A=0\},
	\ese
	which leads to the expression
	\bse
	\bb = (1-\pi)^{-1}p_{0\cdot0} \left[ E\{\m(\x)\mid 1,0\}, E\{\m(\x)\mid 0,0\}\right]^{-1} 
	E\{\m(\x)\mid R=0,A=0\},
	\ese
	provided that the $2\times 2$ matrix $\left[ E\{\m(\x)\mid 1,0\},
	E\{\m(\x)\mid 0,0\}\right]$ is invertible.
	To use the idea of moment matching, one has the flexibility of choosing different moments $\m(\x)$.
	Certainly, a further research question of interest is to identify the optimal choice of this moment function, say, $\m_{\rm opt}(\x)$, by borrowing the semiparametric techniques \citep{bickel1993efficient, tsiatis2006semiparametric}.

    \subsection{Proofs and More Details in Section~\ref{sec:theory}}\label{sec:supp:theory}
    
	
	We define the Rademacher complexity \citep{bartlett2002rademacher} that has been frequently
	used in machine learning literature to establish a generalization bound. Instead of considering
	the Rademacher complexity on $\mF$ we define the class of weighted losses $\mG(\ell,\mF)=[w(x,y)\ell\{g(x),y\}:g\in \mF]$ and $n\in N$ we define its Rademacher complexity measure as 
	\bse
	\mR_n(\mG):= E_{u_i,v_i} \left(E_{\xi_i}\left[\sup_{h\in \mF}\frac{1}{n}\sum_{i=1}^n\xi_iw(u_i,v_i)\ell\{g(u_i),v_i\}\right]\right),
	\ese
	where $\{\xi_i\}_{i=1}^n$ are i.i.d. Rademacher random variables, taking values $\pm 1$ with equal probability $1/2$.

	\begin{proof}[Proof of Theorem \ref{thm:conver}]
		For a probabilistic classifier: $\{\xi_0(\x),1-\xi_0(\x)\}:\mX\rightarrow \Delta^2$, and the parameter 
		$\bb\trans=(\beta_{10},\beta_{00})$ and $b_1=\pr(Y=1|R=1,A=0)$, we define the centered logit function $\f:\mX\rightarrow \R^2$ as $f_0(\x)=\log\xi_0(\x)-\frac{1}{2}[\log\xi_0(\x)+\log\{1-\xi_0(\x)\}]$ and $f_1(\x)=\log\{1-\xi_0(\x)\}-\frac{1}{2}[\log\xi_0(\x)+\log\{1-\xi_0(\x)\}]$. We define the functions $\mu(f_0,b_1)=\xi_0(\x)b_1^{-1}-\{1-\xi_0(\x)\}(1-b_1)^{-1}$ and
		$\omega(f_0,b_1,\beta_{10},\varrho)=\xi_0(\x)b_1^{-1}\beta_{10}+\{1-\xi_0(\x)\}(1-b_1)^{-1}(\varrho-\beta_{10})$, and notice that the objective is
		\bse
		\wh{L}(f_0,b_1,\beta_{10}, \varrho)=\wh{E}\left\{\log \omega(f_0,b_1,\beta_{10},\varrho)|R=0,A=0\right\},
		\ese
		whereas the true objective is 
		\bse
		L(f_0,b_1,\beta_{10}, \varrho)=E\left\{\log \omega(f_0,b_1,\beta_{10},\varrho)|R=0,A=0\right\},
		\ese
		
		We see that the first-order optimality conditions in estimating $\wh\beta_{10}$ are
		\be\label{Taylor}
		0&=&\partial_{\beta_{10}}\wh{L}(\wh{f_0},\wh{b}_1,\wh{\beta}_{10},\wh\varrho)\\ \nonumber
		&=&\partial_{\beta_{10}}\left[\wh{E}\left\{\log \omega(\wh{f}_0,\wh{b}_1,\wh\beta_{10},\wh\varrho)\middle|R=0,A=0\right\}\right]\\ \nonumber &=&\wh{E}\left[\frac{\partial_{\beta_{10}}\{\omega(\wh{f}_0,\wh{b}_1,\wh\beta_{10},\wh\varrho)\}}{\omega(\wh{f}_0,\wh{b}_1,\wh\beta_{10},\wh\varrho)}\middle|R=0,A=0\right].
		\ee
		Similarly, the first order optimality condition at truth (for $\beta_{10}$) are
		\bse
		0&=&\partial_{\beta_{10}}{L}({f_0},{b_1},{\beta}_{10},\varrho)\\ \nonumber
		&=&\partial_{\beta_{10}}\left[{E}\left\{\log \omega({f}_0,{b}_1,\beta_{10},\varrho)\middle|R=0,A=0\right\}\right]\\ \nonumber &=&{E}\left[\frac{\partial_{\beta_{10}}\{\omega(f_0,b_1,\beta_{10},\varrho)\}}{\omega(f_0,b_1,\beta_{10},\varrho)}\middle|R=0,A=0\right].
		\ese
		We decompose (\ref{Taylor}) using the Taylor expansion and obtain:
		\bse
		0=\partial_{\beta_{10}}\wh{L}(f_0,\wh{b}_1,\wh\beta_{10},\wh\varrho)+\langle \wh{f}_0-f_0,\partial_{f_0} \partial_{\beta_{10}}\wh{L}(\wt{f}_0,\wh{b}_1,\wh\beta_{10},\wh\varrho)\rangle
		\ese
		where $\wt{f}_0$ is a function in the bracket $[f_0,\wh{f}_0]$, i.e. for every $\x$, $\wt{f}_0(\x)$ is a number between $\wh{f}_0(\x)$ and $f_0(\x)$.\\
		{\bf{Bound on $\langle \wh{f}_0-f_0,\partial_{f_0} \partial_{\beta_{10}}\wh{L}(\wt{f}_0,\wh{b}_1,\wh\beta_{10},\wh\varrho)\rangle$:}}

		To bound the term, we define $\zeta_0=\wh{f}_0-f_0$ and notice that
		\bse
		&&\langle \zeta_0, \partial_{f_0} \partial_{\beta_{10}}\wh{L}(\wt{f}_0,\wh{b}_1,\wh\beta_{10},\wh\varrho) \rangle\\
		&=&\left\langle \zeta_0, \partial_{f_0}\left[ \wh{E}\left\{\frac{\mu(\wt{f}_0,\wh{b}_1)}{\omega(\wt{f}_0,\wh{b}_1,\wh\beta_{10},\wh\varrho)}\middle|R=0,A=0\right\}\right] \right\rangle\\
		&=&\wh{E}\left(\zeta_0 \frac{2\wt\xi_0(1-\wt\xi_0)}{\omega(\wt{f}_0,\wh{b}_1,\wh\beta_{10},\wh\varrho)}
		\left[\wh{b}_1^{-1}+(1-\wh{b}_1)^{-1}\right.\right.\\
&&\left.\left.-\frac{\mu(\wt{f}_0,\wh{b}_1)}{\omega(\wt{f}_0,\wh{b}_1,\wh\beta_{10},\wh\varrho)}
\left\{\wh{b}_1^{-1}\wh{\beta}_{10}-(1-\wh{b}_1)^{-1}(\wh\varrho-\wh\beta_{10})\right\}\right]\middle|R=0,A=0\right).
		\ese
		The derivative in third equality in the above display is calculated in Lemma \ref{Derivatives}.
		Assume $\varrho-\epsilon>\beta_{10}>\epsilon>0$ and $1-\epsilon_1>b_1>\epsilon_1>0$, i.e., there exist a $c_1>0$ such that \\
$\left|\frac{2\wt\xi_0(1-\wt\xi_0)}{\omega(\wt{f}_0,\wh{b}_1,\wh\beta_{10},\wh\varrho)}
		\left[\wh{b}_1^{-1}+(1-\wh{b}_1)^{-1}-\frac{\mu(\wt{f}_0,\wh{b}_1)}{\omega(\wt{f}_0,\wh{b}_1,\wh\beta_{10},\wh\varrho)}
\left\{\wh{b}_1^{-1}\wh{\beta}_{10}-(1-\wh{b}_1)^{-1}(\wh\varrho-\wh\beta_{10})\right\}\right]\right|<c_1$. This implies the followings: we have
		\bse
		&&\left|\wh{E}\left(\zeta_0 \frac{2\wt\xi_0(1-\wt\xi_0)}{\omega(\wt{f}_0,\wh{b}_1,\wh\beta_{10},\wh\varrho)}
		\left[\wh{b}_1^{-1}+(1-\wh{b}_1)^{-1}\right.\right.\right.\\
&&\left.\left.\left.-\frac{\mu(\wt{f}_0,\wh{b}_1)}{\omega(\wt{f}_0,\wh{b}_1,\wh\beta_{10},\wh\varrho)}
\left\{\wh{b}_1^{-1}\wh\beta_{10}-(1-\wh{b}_1)^{-1}(\wh\varrho-\wh\beta_{10})\right\}\right]\middle|R=0,A=0\right)\right|\\
&\leq& c_1\wh{E}\{|\zeta_0(\x)||R=0,A=0\}.
		\ese
		It follows from Assumption \ref{ass1}
		with probability at least $1-\delta$ it holds $\sup_{i\in [n_{0\cdot 0}]}\|\wh\f(\x_{i})-\f(\x_{i})\|_2\leq c r_{n_{1\cdot 0}}\sqrt{\log(n_{0\cdot 0})\log(1/\delta)}$, we conclude that
		\bse
		|\langle \zeta_0, \partial_{f_0} \partial_{\beta_{10}}\wh{L}(\wt{f}_0,\wh{b}_1,\wh\beta_{10},\wh\varrho) \rangle|\leq cc_1r_{n_{1\cdot 0}}\sqrt{\log(n_{0\cdot 0})\log(1/\delta)}
		\ese
		holds with probability at least $1-\delta$.
		
		{\bf{Bound on $\partial_{\beta_{10}}\wh{L}({f}_0,\wh{b}_1,\wh\beta_{10},\wh\varrho)-
				\partial_{\beta_{10}}\wh{L}({f}_0,{b}_1,\wh\beta_{10},\wh\varrho)$}}.
		
		Using the taylor expansion, we have
		\bse
		&&\partial_{\beta_{10}}\wh{L}({f}_0,\wh{b}_1,\wh\beta_{10},\wh\varrho)-
				\partial_{\beta_{10}}\wh{L}({f}_0,{b}_1,\wh\beta_{10},\wh\varrho)
		=\langle \wh b_1-b_1, \partial_{b_1}\partial_{\beta_{10}}\wh{L}({f}_0,\wt{b}_1,\wh\beta_{10},\wh\varrho)\rangle \\
&=&\left\langle \wh b_1-b_1, \partial_{b_1}\left[ \wh{E}\left\{\frac{\mu({f}_0,\wt{b}_1)}{\omega({f}_0,\wt{b}_1,\wh\beta_{10},\wh\varrho)}\middle|R=0,A=0\right\}\right]\right\rangle\\
		&=&\wh{E}
		\left\{(\wh b_1-b_1) \left[\frac{-\xi_0(\x)\wt{b}_1^{-2}-\{1-\xi_0(\x)\}(1-\wt{b}_1)^{-2}}{\omega({f}_0,\wt{b}_1,\wh\beta_{10},\wh\varrho)}\right.\right.\\
&&\left.\left.-\frac{\mu(f_0,\wt{b}_1)}{\omega({f}_0,\wt{b}_1,\wh\beta_{10},\wh\varrho)}
\frac{-\xi_0(\x)\wt{b}_1^{-2}\wh{\beta}_{10}+\{1-\xi_0(\x)\}(1-\wt{b}_1)^{-2}(\wh{\varrho}-\wh{\beta}_{10})}{\omega({f}_0,\wt{b}_1,\wh\beta_{10},\wh\varrho)}\right]\right\}.
		\ese
		Assume $\varrho-\epsilon>\beta_{10}>\epsilon>0$ and $1-\epsilon_1>b_1>\epsilon_1>0$, i.e., there exist a $c_2>0$ such that $\left|\left[\frac{-\xi_0(\x)\wt{b}_1^{-2}-\{1-\xi_0(\x)\}(1-\wt{b}_1)^{-2}}{\omega({f}_0,\wt{b}_1,\wh\beta_{10},\wh\varrho)}
-\frac{\mu(f_0,\wt{b}_1)}{\omega({f}_0,\wt{b}_1,\wh\beta_{10},\wh\varrho)}
\frac{-\xi_0(\x)\wt{b}_1^{-2}\wh{\beta}_{10}+\{1-\xi_0(\x)\}(1-\wt{b}_1)^{-2}(\wh{\varrho}-\wh{\beta}_{10})}{\omega({f}_0,\wt{b}_1,\wh\beta_{10},\wh\varrho)}\right]\right|<c_2$. This implies the followings: we have
		\bse
		&&\Bigg|\wh{E}
		\left\{(\wh b_1-b_1) \left[\frac{-\xi_0(\x)\wt{b}_1^{-2}-\{1-\xi_0(\x)\}(1-\wt{b}_1)^{-2}}{\omega({f}_0,\wt{b}_1,\wh\beta_{10},\wh\varrho)}\right.\right.\\
&&\left.\left.-\frac{\mu(f_0,\wt{b}_1)}{\omega({f}_0,\wt{b}_1,\wh\beta_{10},\wh\varrho)}
\frac{-\xi_0(\x)\wt{b}_1^{-2}\wh{\beta}_{10}+\{1-\xi_0(\x)\}(1-\wt{b}_1)^{-2}(\wh{\varrho}-\wh{\beta}_{10})}{\omega({f}_0,\wt{b}_1,\wh\beta_{10},\wh\varrho)}\right]\right\}\Bigg|\\
		&\leq& c_2 |\wh b_1-b_1|.
		\ese
		We apply Hoeffing’s concentration inequality for a sample mean of i.i.d. sub-gaussian random variable $Y_i$ and obtain a $c_3>0$  such that for any $\delta>0$ with probability at least $1-\delta$ it holds
		\bse
		|\wh b_1-b_1|=|\wh\pr(Y = 1|R = 1, A = 0)-\pr(Y = 1|R = 1, A = 0)|\leq c_3\sqrt{\frac{\log(1/\delta)}{n_{1\cdot 0}}}.
		\ese

		{\bf{Bound on $\partial_{\beta_{10}}\wh{L}({f}_0,{b}_1,\wh\beta_{10},\wh\varrho)-
				\partial_{\beta_{10}}\wh{L}({f}_0,{b}_1,\wh\beta_{10},\varrho)$}}.
		
		Using the taylor expansion, we have
		\bse
		&&\partial_{\beta_{10}}\wh{L}({f}_0,{b}_1,\wh\beta_{10},\wh\varrho)-\partial_{\beta_{10}}\wh{L}({f}_0,{b}_1,\wh\beta_{10},\varrho)
		=\langle \wh\varrho-\varrho, \partial_{\varrho}\partial_{\beta_{10}}\wh{L}(f_0, {b}_1,\wh\beta_{10},\wt\varrho)\rangle \\
		&=&\wh{E}
		\left[(\wh\varrho-\varrho) \frac{-\mu(f_0,b_1)\{1-\xi_0(\x)\}(1-b_1)^{-1}}{\omega^2({f}_0,{b}_1,\wh\beta_{10},\wt\varrho)}\middle|R=0,A=0\right].
		\ese
		Assume $\varrho-\epsilon>\beta_{10}>\epsilon>0$ and $1-\epsilon_1>b_1>\epsilon_1>0$, i.e., there exist a $c_4>0$ such that $\left|\frac{-\mu(f_0,b_1)\{1-\xi_0(\x)\}(1-b_1)^{-1}}{\omega^2({f}_0,{b}_1,\wh\beta_{10},\wt\varrho)}\right|<c_4$. This implies the followings: we have
		\bse
		\left|\wh{E}
		\left[(\wh\varrho-\varrho) \frac{-\mu(f_0,b_1)\{1-\xi_0(\x)\}(1-b_1)^{-1}}{\omega^2({f}_0,{b}_1,\wh\beta_{10},\wt\varrho)}\middle|R=0,A=0\right]
\right|
		\leq c_4 |\wh\varrho-\varrho|.
		\ese
		We apply Hoeffing’s concentration inequality for a sample mean of i.i.d. sub-gaussian random variable $A_i$ and obtain a $c_5>0$  such that for any $\delta>0$ with probability at least $1-\delta$ it holds
		\bse
		|\wh\varrho-\varrho|=|\wh\pr(A = 0|R = 0)-\pr(A = 0|R=0)|\leq c_5\sqrt{\frac{\log(1/\delta)}{n_{0}}}.
		\ese

		{\bf{The term $\partial_{\beta_{10}}\wh{L}(f_0,\wh{b}_1,\wh\beta_{10}, \wh\varrho)$:}} We  have
		\bse
		&&\partial_{\beta_{10}}\wh{L}(f_0,\wh{b}_1,\wh\beta_{10}, \wh\varrho)-
		\partial_{\beta_{10}}\wh{L}(f_0,{b}_1,\wh\beta_{10}, \varrho)+\partial_{\beta_{10}}\wh{L}(f_0,{b}_1,\wh\beta_{10},\varrho)\\
&=&\partial_{\beta_{10}}\wh{L}(f_0,{b}_1,\wh\beta_{10},\varrho)+O_p(|\wh{b}_1-b_1|+|\wh\varrho-\varrho|).
		\ese
		Now, we study the term $\partial_{\beta_{10}}\wh{L}(f_0,{b}_1,\wh\beta_{10},\varrho)$, use strong convexity of $-L(f_0,{b}_1,\beta_{10},\varrho)$ with $\beta_{10}$ and the convergence of the loss that
		\bse
		\sup_{\beta_{10}\in (0,\varrho)}|\wh{L}(f_0,{b}_1,\beta_{10},\varrho)-L(f_0,{b}_1,\beta_{10},\varrho)|
		\xrightarrow{n_{0\cdot 0}\rightarrow \infty}0
		\ese
		for $\beta_{10}\in (0,\varrho)$ in \cite{wellner2013weak}(see Corollary 3.2.3) to conclude that $\wh\beta_{10}\rightarrow \beta_{10}$ in probability and hence $\wh\beta_{10}$ is a consistent estimator for $\beta_{10}$.
		
		Following the consistency of $\wh\beta_{10}$ we see that for sufficiently large $n_{0\cdot 0}$, we have $|\wh\beta_{10}-\beta_{10}|\leq \delta_{\beta}$($\delta_{\beta}$ is chosen bound by $\frac{\beta_{10}}{2} \wedge\frac{\varrho-\beta_{10}}{2}$) with probability at least $1-\delta$
		and on the event it holds: $\wh\beta_{10}\in[\beta_{10}-\delta_{\beta},\beta_{10}+\delta_{\beta}]$. We define empirical process
		\bse
		Z_{n_{0\cdot 0}}=\sup_{\beta\in[\beta_{10}-\delta_{\beta},\beta_{10}+\delta_{\beta}]}
		|\partial_{\beta}\wh{L}(f_0,{b}_1,\beta,\varrho)-\partial_{\beta}L(f_0,{b}_1,\beta,\varrho)|
		\ese
		for which we shall provide a high probability upper bound. We denote $Z_{n_{0\cdot 0}}(\beta)=\partial_{\beta}\wh{L}(f_0,{b}_1,\beta,\varrho)-\partial_{\beta}L(f_0,{b}_1,\beta,\varrho)$ and notice that
		\bse
		&&\partial_{\beta}\wh{L}(f_0,{b}_1,\beta,\varrho)-\partial_{\beta}L(f_0,{b}_1,\beta,\varrho)\\
		&=&\wh{E}\left\{\frac{\mu(f_0,b_1)}{\omega(f_0,b_1,\beta,\varrho)}\middle|R=0,A=0\right\}
		-{E}\left\{\frac{\mu(f_0,b_1)}{\omega(f_0,b_1,\beta,\varrho)}\middle|R=0,A=0\right\}
		:=A(\beta)
		\ese
		where to bound $A(\beta)$ we notice that $\frac{\mu(f_0,b_1)}{\omega(f_0,b_1,\beta,\varrho)}$ are i.i.d. and bounded by $c_0$($\beta^{-1}+(\varrho-\beta)^{-1}\leq \frac{2}{\beta_{10}}+\frac{2}{\varrho-\beta_{10}}\leq c_0$ for all $\x\in\mX$) and hence sub-gaussian. We apply Hoeffding's concentration inequality for a sample mean if i.i.d. sub-gaussian random variables and obtain a constant $c_6>0$ such that for any $\delta>0$ with probability at least $1-\delta$ it holds
		\bse
		A(\beta)&=&\wh{E}\left\{\frac{\mu(f_0,b_1)}{\omega(f_0,b_1,\beta,\varrho)}\middle|R=0,A=0\right\}
		-{E}\left\{\frac{\mu(f_0,b_1)}{\omega(f_0,b_1,\beta,\varrho)}\middle|R=0,A=0\right\}\\
        &&\leq c_0 c_6\sqrt{\frac{\log(1/\delta)}{n_{0\cdot 0}}}.
		\ese
		Use chained arguments for $\ell_1$ with interval length $2\delta_{\beta}$ we obtain a uniform bound as the following: there exists a constant $c_7>0$ such that for any $\delta>0$ with probability at least $1-\delta$ if it holds
		\bse
		\sup_{\beta\in[\beta_{10}-\delta_{\beta},\beta_{10}+\delta_{\beta}]}A(\beta)\leq c_0 c_6c_7\sqrt{\frac{\log(1/\delta)}{n_{0\cdot 0}}}.
		\ese
		Therefore, with probability at least $1-\delta$, we have
		\bse
		Z_{n_{0\cdot 0}}\leq c_0 c_6c_7\sqrt{\frac{\log(1/\delta)}{n_{0\cdot 0}}}.
		\ese
		Returning to the first order optimality condition for estimating $\wh\beta_{10}$ we notice that
		\bse
		0&=&(\wh\beta_{10}-\beta_{10})\left\{\partial_{\beta_{10}}\wh{L}(f_0,\wh{b}_1,\wh\beta_{10},\wh\varrho)+\langle \wh f_0-f_0,\partial_{f_0} \partial_{\beta_{10}}\wh{L}(\wt{f}_0,\wh{b}_1,\wh\beta_{10},\wh\varrho)\rangle\right\}\\ &=&(\wh\beta_{10}-\beta_{10})\left\{\partial_{\beta_{10}}\wh{L}(f_0,\wh{b}_1,\wh\beta_{10},\wh\varrho)-\partial_{\beta_{10}}\wh{L}(f_0,b_1,\wh\beta_{10},\varrho)
+\partial_{\beta_{10}}\wh{L}(f_0,b_1,\wh\beta_{10},\varrho)\right.\\
&&\left.+\langle \wh{f}_0-f_0,\partial_{f_0} \partial_{\beta_{10}}\wh{L}(\wt{f}_0,\wh{b}_1,\wh\beta_{10},\wh\varrho)\rangle\right\}\\
		&=&(\wh\beta_{10}-\beta_{10}) \partial_{\beta_{10}}L(f_0,b_1,\wh\beta_{10},\varrho)\\	&&+(\wh\beta_{10}-\beta_{10})\left\{\partial_{\beta_{10}}\wh{L}(f_0,\wh{b}_1,\wh\beta_{10},\wh\varrho)-\partial_{\beta_{10}}\wh{L}(f_0,b_1,\wh\beta_{10},\varrho)+ Z_{n_{0\cdot 0}}(\wh\beta_{10})\right.\\
&&\left.+\langle \wh{f}_0-f_0,\partial_{f_0} \partial_{\beta_{10}}\wh{L}(\wt{f}_0,\wh{b}_1,\wh\beta_{10},\wh\varrho)\rangle\right\}.
		\ese
		We combine it with the first order optimality condition for $\bb$ to obtain
		\bse
		&&(\wh\beta_{10}-\beta_{10})\left\{\partial_{\beta_{10}}L(f_0,b_1,\wh\beta_{10},\varrho)-\partial_{\beta_{10}}L(f_0,b_1,\beta_{10},\varrho)\right\}\\
&&+(\wh\beta_{10}-\beta_{10})\left\{\partial_{\beta_{10}}\wh{L}(f_0,\wh{b}_1,\wh\beta_{10},\wh\varrho)-\partial_{\beta_{10}}\wh{L}(f_0,b_1,\wh\beta_{10},\varrho)+ Z_{n_{0\cdot 0}}(\wh\beta_{10})\right.\\
&&\left.+\langle \wh{f}_0-f_0,\partial_{f_0} \partial_{\beta_{10}}\wh{L}(\wt{f}_0,\wh{b}_1,\wh\beta,\wh\varrho)\rangle\right\}=0,
		\ese
		which can be rewritten as
		\be\label{q1}
		&&-(\wh\beta_{10}-\beta_{10})\left\{\partial_{\beta_{10}}L(f_0,b_1,\wh\beta_{10},\varrho)-\partial_{\beta_{10}}L(f_0,b_1,\beta_{10},\varrho)\right\}\\ \nonumber
		&=&(\wh\beta_{10}-\beta_{10})\left\{\partial_{\beta_{10}}\wh{L}(f_0,\wh{b}_1,\wh\beta_{10},\wh\varrho)-\partial_{\beta_{10}}\wh{L}(f_0,b_1,\wh\beta_{10},\varrho)+ Z_{n_{0\cdot 0}}(\wh\beta_{10})\right.\\ \nonumber
&&\left.+\langle \wh{f}_0-f_0,\partial_{f_0} \partial_{\beta_{10}}\wh{L}(\wt{f}_0,\wh{b}_1,\wh\beta_{10},\wh\varrho)\rangle\right\}.
		\ee
		Using the strong convexity of function $-L$ at $\beta_{10}$, we obtain that the left-hand side in the above equation
		is lower bounded as
		\be\label{A15}
		-(\wh\beta_{10}-\beta_{10}) \left\{\partial_{\beta_{10}}L(f_0,b_1,\wh\beta_{10},\varrho)-\partial_{\beta_{10}}L(f_0,b_1,\beta_{10},\varrho)\right\}\geq \mu(\wh\beta_{10}-\beta_{10})^2.
		\ee
		Let $\mE$ be the event on which the following hold:
		\begin{itemize}[left=0pt]
			\item $|\wh\beta_{10}-\beta_{10}|\leq \delta_{\beta}$.
			\item $|\langle \wh{f}_0-f_0,\partial_{f_0} \partial_{\beta_{10}}\wh{L}(\wt{f}_0,\wh{b}_1,\wh\beta_{10}, \wh\varrho)\rangle|\leq cc_1r_{n_{1\cdot 0}}\sqrt{\log(n_{0\cdot 0})\log(1/\delta)}$.
			\item $Z_{n_{0\cdot 0}}\leq c_0 c_6c_7\sqrt{\frac{\log(1/\delta)}{n_{0\cdot 0}}}$.
			\item $|\partial_{\beta_{10}}\wh{L}(f_0,\wh{b}_1,\wh\beta_{10}, \wh\varrho)-\partial_{\beta_{10}}\wh{L}(f_0,b_1,\wh\beta_{10}, \varrho)|\leq (c_2c_3+c_4c_5)\left\{\sqrt{\frac{\log(1/\delta)}{n_{1\cdot 0}}}+\sqrt{\frac{\log(1/\delta)}{n_{0}}}\right\}$.
		\end{itemize}
		We notice that the event $\mE$ has probability $1-5\delta$.
		Under the event there exists a $c_8 > 0$ such
		that the right-hand side in (\ref{q1}) is upper bounded as
		\be\label{A16}
        &&\left|(\wh\beta_{10}-\beta_{10})\left\{\partial_{\beta_{10}}\wh{L}(f_0,\wh{b}_1,\wh\beta_{10},\wh\varrho)-\partial_{\beta_{10}}\wh{L}(f_0,b_1,\wh\beta_{10},\varrho)+ Z_{n_{0\cdot 0}}(\wh\beta_{10})\right.\right.\\ \nonumber
&&\left.\left.+\langle \wh{f}_0-f_0,\partial_{f_0} \partial_{\beta_{10}}\wh{L}(\wt{f_0},\wh{b}_1,\wh\beta_{10},\varrho)\rangle\right\}\right|\\ \nonumber
		&\leq& |\wh\beta_{10}-\beta_{10}|\left\{|\partial_{\beta_{10}}\wh{L}(f_0,\wh{b}_1,\wh\beta_{10},\wh\varrho)-\partial_{\beta_{10}}\wh{L}(f_0,b_1,\wh\beta_{10},\varrho)|+ |Z_{n_{0\cdot 0}}(\wh\beta_{10})|\right.\\ \nonumber
       &&\left.+|\langle \wh{f}_0-f_0,\partial_{f_0} \partial_{\beta_{10}}\wh{L}(\wt{f}_0,\wh{b}_1,\wh\beta_{10},\wh\varrho)\rangle|\right\}\\ \nonumber
		&\leq &c_8\left\{r_{n_{1\cdot 0}}\sqrt{\log(n_{0\cdot 0})\log(1/\delta)}
		+\sqrt{\frac{\log(1/\delta)}{n_{1\cdot 0}}}+\sqrt{\frac{\log(1/\delta)}{n_{0}}}+\sqrt{\frac{\log(1/\delta)}{n_{0\cdot 0}}}\right\}|\wh\beta_{10}-\beta_{10}|.
		\ee
		Combining the bounds (\ref{A15}) and (\ref{A16}) for left and right hand sides we obtain a $c_{10} > 0$ such that
		on the event $\mE$ it holds
		\bse
		|\wh\beta_{10}-\beta_{10}|\leq c_{10}\left\{r_{n_{1\cdot 0}}\sqrt{\log(n_{0\cdot 0})\log(1/\delta)}
		+\sqrt{\frac{\log(1/\delta)}{n_{1\cdot 0}}}+\sqrt{\frac{\log(1/\delta)}{n_{0}}}+\sqrt{\frac{\log(1/\delta)}{n_{0\cdot 0}}}\right\}.
		\ese
 
        Further, since $A_i$ is bound random variable, then we can obtain a constant $c_9>0$ such that for any $\delta>0$ with probability at least $1-\delta$ it holds 
        \bse
        |\wh\beta_{00}-\beta_{00}|=|(\wh\varrho-\wh\beta_{10})-(\varrho-\beta_{10})|
        \leq |\wh\varrho-\varrho|+|\wh\beta_{10}-\beta_{10}|\leq c_9\sqrt{\frac{\log(1/\delta)}{n_{0}}}+|\wh\beta_{10}-\beta_{10}|.
        \ese
        In summary, we have a constant $c_{10}>0$ such that for any $\delta>0$ with probability at least $1-6\delta$ it holds
        \bse
        \|\wh\bb-\bb\|_1\leq c_{10}\left\{r_{n_{1\cdot 0}}\sqrt{\log(n_{0\cdot 0})\log(1/\delta)}
		+\sqrt{\frac{\log(1/\delta)}{n_{1\cdot 0}}}+\sqrt{\frac{\log(1/\delta)}{n_{0}}}++\sqrt{\frac{\log(1/\delta)}{n_{0\cdot 0}}}\right\}.
        \ese
	\end{proof}

	\begin{lemma}\label{Derivatives}(Derivatives).The following equations hold:
		\begin{itemize}[left=0pt]
			\item $\partial_{f_0}(\xi_0)=2\xi_0(1-\xi_0)$;
			\item $\partial_{f_0}\{\mu(f_0,b_1)\}=\partial_{f_0}(\xi_0)\{b_1^{-1}+(1-b_1)^{-1}\}$;
			\item $\partial_{f_0}\{\omega(f_1,b_1,\beta_{10},\varrho)\}=\partial_{f_0}(\xi_0)\{b_1^{-1}\beta_{10}-(1-b_1)^{-1}(\varrho-\beta_{10})\}$;
			\item $\partial_{f_0}\left\{\frac{\mu(f_0,b_1)}{\omega(f_0,b_1,\beta_{10},\varrho)}\right\}=\frac{2\xi_0(1-\xi_0)}{\omega(f_0,b_1,\beta_{10},\varrho)}
\left[b_1^{-1}+(1-b_1)^{-1}-\frac{\mu(f_0,b_1)\{b_1^{-1}\beta_{10}-(1-b_1)^{-1}(\varrho-\beta_{10})\}}{\omega(f_0,b_1,\beta_{10},\varrho)}\right]$.
		\end{itemize}
	\end{lemma}
	
	\begin{proof}
		\bse
		&&\partial_{f_0}\xi_0=\partial_{f_0}\left(\frac{e^{f_{0}}}{ e^{f_{0}}+e^{-f_0}}\right)
		=\frac{e^{f_{0}}\left(e^{f_{0}}+e^{-f_{0}}\right)-e^{f_{0}}\left(e^{f_{0}}-e^{-f_{0}}\right)}{\left(e^{f_{0}}+e^{-f_0}\right)^2}
		=2\xi_0(1-\xi_0),\\
		&&\partial_{f_0}\{\mu(f_0,b_1)\}=\partial_{f_0}(\xi_0)\{b_1^{-1}+(1-b_1)^{-1}\},\\
		&&\partial_{f_0}\{\omega(f_0,b_1,\beta_{10},\varrho)\}=\partial_{f_0}(\xi_0)b_1^{-1}\beta_{10}-\partial_{f_0}(\xi_0)(1-b_1)^{-1}(\varrho-\beta_{10})\\
&&=\partial_{f_0}(\xi_0)\{b_1^{-1}\beta_{10}-(1-b_1)^{-1}(\varrho-\beta_{10})\}.
		\ese
		Thus, 
		\bse
		&&\partial_{f_0}\left\{\frac{\mu(f_0,b_1)}{\omega(f_0,b_1,\beta_{10},\varrho)}\right\}\\	&=&\frac{\partial_{f_0}(\xi_0)[\{b_1^{-1}+(1-b_1)^{-1}\}\omega(f_0,b_1,\beta_{10},\varrho)
-\mu(f_0,b_1)\{b_1^{-1}\beta_{10}-(1-b_1)^{-1}(\varrho-\beta_{10})\}]}
{\omega^2(f_0,b_1,\beta_{10},\varrho)}\\
		&=&\frac{\partial_{f_0}(\xi_0)}{\omega(f_0,b_1,\beta_{10},\varrho)}
\left[b_1^{-1}+(1-b_1)^{-1}-\frac{\mu(f_0,b_1)\{b_1^{-1}\beta_{10}-(1-b_1)^{-1}(\varrho-\beta_{10})\}}{\omega(f_0,b_1,\beta_{10},\varrho)}\right]\\	
        &=&\frac{2\xi_0(1-\xi_0)}{\omega(f_0,b_1,\beta_{10},\varrho)}
\left[b_1^{-1}+(1-b_1)^{-1}-\frac{\mu(f_0,b_1)\{b_1^{-1}\beta_{10}-(1-b_1)^{-1}(\varrho-\beta_{10})\}}{\omega(f_0,b_1,\beta_{10},\varrho)}\right].
		\ese
	\end{proof}

	\begin{proof}[Proof of Proposition \ref{pre}]
		Define $w(y)=\frac{pr(y|A=0, R=0)}{pr(y|A=0,R=1)}$,
		\bse
		\mL_0(h)&=&\E[\ell\{h(\X),Y\}|R=0,A=0]\\
		&=&\int \ell\{h(\X),Y\}p(\x,y|A=0,R=0)\d\x\d y\\
		&=&\int \ell\{h(\X),Y\}\frac{p(\x,y|A=0,R=0)}{p(\x,y|A=0,R=1)}p(\x,y|A=0,R=1)\d\x\d y\\
		&=&\int \ell\{h(\x),y\}\frac{pr(y|A=0, R=0)}{pr(y|A=0,R=1)}\pr(\x,y|A=0,R=1)\d\x\d y\\
		&=&\E\left[\ell\{h(\X),Y\}w(Y)\big|A=0,R=1\right]=:\L_1(h,w).
		\ese
		
		Let $\mL_1(h,w)=\E\left[\ell\{h(\X),Y\}w(Y)\big|A=0,R=1\right]$, then we have
		\be\label{WW}\nonumber
		\mL_0(\wh h)-\mL_0(h)&=&\mL_1(\wh h,w)-\mL_1(h,w)\\
		&=&\underbrace{\mL_1(\wh h,w)-\wh\mL_1(\wh h,w)}_{(a)}
		+\underbrace{\wh\mL_1(\wh h,w)-\wh\mL_1(\wh h,\wh{w})}_{(b)}\\ \nonumber
		&&+\underbrace{\wh\mL_1(\wh h,\wh{w})-\wh\mL_1(h,\wh{w})}_{\leq 0}
		+\underbrace{\wh\mL_1(h,\wh{w})-\wh\mL_1(h,w)}_{(c)}
		+\underbrace{\wh\mL_1(h,w)-\mL_1(h,w)}_{(d)},
		\ee
		where $\wh h\equiv \wh h_{\wh w}$.

		{\bf{Uniform bound on (a)}} To control (a) in (\ref{WW}) we establish a concentration bound on the following
		generalization error
		\bse
		&&\sup_{g\in \mF}\{\mL_1(g,w)-\wh\mL_1(g,w)\}\\
		&=&\sup_{g\in \mF}\left\{E\left[\ell\{g(\X),Y\}w(Y)\big|A=0,R=1\right]
		-\wh{E}\left[\ell\{g(\X),Y\}w(Y)\big|A=0,R=1\right]\right\}\\
		&=&:F(\Z_{1:n_{1\cdot 0}})
		\ese
		where, for $i > 1$ we denote $\Z_{1:i}=(\Z_1,\cdots, \Z_i)$ and $\Z_i= (\X_i, Y_i)$. First, we use a modification of McDiarmid concentration inequality to bound $F(\Z_{1:n_{1\cdot 0}})$ in terms of its expectation and a $O_p(1/\sqrt{n_{1\cdot 0}})$ term, as elucidated in the following lemma.
		\begin{lemma}
			There exists a constant $c_1 > 0$ such that with probability at least $1-\delta$ the following holds
			\be\label{Rad1}
			F(\Z_{1:n_{1\cdot 0}})\leq E\{F(\Z_{1:n_{1\cdot 0}})\}+c_1\sqrt{\frac{\log(1/\delta)}{n_{1\cdot 0}}}.
			\ee
			The proof is similar to Lemma A.3 of \cite{maity2022understanding}, so we omit it.
		\end{lemma}
		Next, we use a symmetrization argument (see \cite{wellner2013weak}, Chapter 2, Lemma 2.3.1) to bound the expectation $E\{F(\Z_{1:n_{1\cdot 0}})\}$ by the Rademacher complexity of the hypothesis class $\mG$, i.e.,
		\be\label{Rad2}
		E\{F(\Z_{1:n_{1\cdot 0}})\}\leq 2 \mR_{n_{1\cdot 0}}(\mG).
		\ee
		Combining (\ref{Rad1}) and (\ref{Rad2}) we obtain
		\be\label{(a)}
		(a)\leq 2 \mR_{n_{1\cdot 0}}(\mG)+c_1\sqrt{\frac{\log(1/\delta)}{n_{1\cdot 0}}}
		\ee
		with probability at least $1-\delta$.

		{\bf{Uniform bound on (b) and (c)}} Denoting $\Z_i=(\X_i,Y_i)$ and $\ell_g(\Z_i)=\ell\{g(\X_i),Y_i\}$ we notice that
		for any $g\in \mF$ we have

		\bse
		& & |\wh\mL_1(g,w)-\wh\mL_1(g, \wh{w})|\\
		&=&|\wh\E\left[\ell\{g(\X),Y\} \left\{w(Y)-\wh w(Y)\right\} \mid A=0,R=1\right] |\leq \frac{\|\ell_g\|_{\infty}}{n_{1\cdot 0}} \sum_{i=1}^{n_{1\cdot 0}} \left| w(y_i)-\wh w(y_i)\right|.
		\ese
		Since $w(y)-\wh w(y)$ is a sub-gaussian random variable, we use sub-gaussian concentration to establish that for some constant $c_2>0$,
		\bse
		\text{for\ any}\ g\in \mF, |\wh\mL_1(g,w)-\wh\mL_1(g, \wh{w})|\leq \|\ell_g\|_{\infty}\left\{\E_Y|w(Y)-\wh w(Y)|+c_2\sqrt{\frac{\log(1/\delta)}{n_{1\cdot 0}}}\right\}
		\ese
		with probability at least $1-\delta$. This provides a simultaneous bound (on the same probability event)
		for both (b) and (c) with $g=\wh{h}$ and $g=h$. Further, by Lemma \ref{con}, for some constants $C_1$ and $C_2$ and any $g\in \mF$, we have
		\be\label{(b)}
		&&|\wh\mL_1(g,w)-\wh\mL_1(g, \wh{w})|\\ \nonumber
		&\leq& \|\ell_g\|_{\infty}\left\{C_1\|\wh\bb-\bb\|_1 +C_2\left(
		\sqrt{\frac{\log(1/\delta)}{n_0}}+\sqrt{\frac{\log(1/\delta)}{n_1}}\right)+c_2\sqrt{\frac{\log(1/\delta)}{n_{1\cdot 0}}}\right\}
		\ee
		with probability at least $1-5\delta$.

		{\bf{Uniform bound on (d)}} We note that
		\bse
		&&\wh\mL_1(h,w)-\mL_1(h,w)\\
		&=&\wh{E}\left[\ell\{h(\X),Y\}w(Y)\big|A=0,R=1\right]
		-E\left[\ell\{h(\X),Y\}w(Y)\big|A=0,R=1\right]
		\ese
		where $[\ell\{h(\X_i),Y_i\}w(Y_i)]_{i=1}^{n_{1\cdot 0}}$ are i.i.d sub-gaussian random variables. Using Hoeffding concentration
		bound we conclude that there exists a constant $c_3>0$ such that for any $\delta>0$ the following holds with probability at least $1-\delta$,
		\be\label{(d)}
		\wh\mL_1(h,w)-\mL_1(h,w)\leq c_3\sqrt{\frac{\log(1/\delta)}{n_{1\cdot 0}}}.
		\ee
		
		Finally, using (\ref{(a)}) on (a) (which is true on an event of probability $\geq 1-\delta$), (\ref{(b)}) on (b) and (c) (simultaneously true on an event of probability $1-5\delta$), and  (\ref{(d)}) on (d) (holds on an event of probability $\geq 1-\delta$) we conclude that with probability at least $1-7\delta$ the following holds
		\bse
		\mL_0(\wh h_{\wh w})-\mL_0(h)\leq 2 \mR_{n_{1\cdot 0}}(\mG)+C B\|\wh\bb- \bb\|_1+c\left\{\sqrt{\frac{\log(1/\delta)}{n_{1\cdot 0}}}+\sqrt{\frac{\log(1/\delta)}{n_0}}+\sqrt{\frac{\log(1/\delta)}{n_1}}\right\}.
		\ese
		where $c=c_1+\|\ell_g\|_{\infty}(C_2+c_2)+c_3$.
	\end{proof}

	\begin{lemma}\label{con}
		Assume $|\beta_{i0}/\alpha_{i0}|\leq B_1$ for any $i=0,1$. There exist constants $C,c_1,c_2$, such that with probability at least $1-4\delta$,
		\bse
		|\wh w(y)-w(y)|\leq CB_1\|\wh\bb-\bb\|_1 +CB_1(c_1+c_2)\left(
		\sqrt{\frac{\log(1/\delta)}{n_0}}+\sqrt{\frac{\log(1/\delta)}{n_1}}\right).
		\ese 
	\end{lemma}
	\begin{proof}
		\bse
		&&|\wh w(y)-w(y)|=\left|\frac{\wh\pr(y|A=0,R=0)}{\wh\pr(y|A=0,R=1)}-\frac{\pr(y|A=0,R=0)}{\pr(y|A=0,R=1)}\right|\\
		&=&\left|\frac{\wh\pr(y,A=0|R=0)}{\wh\pr(y,A=0|R=1)}\frac{\wh\pr(A=0|R=1)}{\wh\pr(A=0|R=0)}
		-\frac{\pr(y,A=0|R=0)}{\pr(y,A=0|R=1)}\frac{\pr(A=0|R=1)}{\pr(A=0|R=0)}\right|\\
		&\leq&\left|\left\{\frac{\wh\pr(y,A=0|R=0)}{\wh\pr(y,A=0|R=1)}-\frac{\pr(y,A=0|R=0)}{\pr(y,A=0|R=1)}\right\}\frac{\wh\pr(A=0|R=1)}{\wh\pr(A=0|R=0)}\right|\\
		&&+\left|\frac{\pr(y,A=0|R=0)}{\pr(y,A=0|R=1)}\left\{\frac{\wh\pr(A=0|R=1)}{\wh\pr(A=0|R=0)}-\frac{\pr(A=0|R=1)}{\pr(A=0|R=0)}   \right\}\right|.
		\ese
		For the first term, we have
		\bse
		&&\left|\frac{\wh\beta_{y0}\alpha_{y0}-\beta_{y0}\wh\alpha_{y0}}{\alpha_{y0}\wh\alpha_{y0}}\frac{n_{1\cdot 0}}{n_{0\cdot 0}}\frac{n_0}{n_1}\right|\\
		&=&\left|\frac{(\wh\beta_{y0}-\beta_{y0})\alpha_{y0}+\beta_{y0}(\alpha_{y0}-\wh\alpha_{y0})}{\alpha_{y0}\wh\alpha_{y0}}\frac{n_{1\cdot 0}}{n_{0\cdot 0}}\frac{n_0}{n_1}\right|\\
		&\leq&\left|\frac{y\alpha_{y0}}{\alpha_{y0}\wh\alpha_{y0}}\frac{n_{1\cdot 0}}{n_{0\cdot 0}}\frac{n_0}{n_1}\right||\wh\beta_{10}-\beta_{10}|
		+\left|\frac{(1-y)\alpha_{y0}}{\alpha_{y0}\wh\alpha_{y0}}\frac{n_{1\cdot 0}}{n_{0\cdot 0}}\frac{n_0}{n_1}\right||\wh\beta_{00}-\beta_{00}|\\
		&&+\left|\frac{y\beta_{y0}}{\alpha_{y0}\wh\alpha_{y0}}\frac{n_{1\cdot 0}}{n_{0\cdot 0}}\frac{n_0}{n_1}\right|(\alpha_{10}-\wh\alpha_{10})
		+\left|\frac{(1-y)\beta_{y0}}{\alpha_{y0}\wh\alpha_{y0}}\frac{n_{1\cdot 0}}{n_{0\cdot 0}}\frac{n_0}{n_1}\right|(\alpha_{00}-\wh\alpha_{00})\\
		&\leq& CB_1(\|\wh\bb-\bb\|_1+|\wh\alpha_{00}-\alpha_{00}|+|\wh\alpha_{10}-\alpha_{10}|).
		\ese
		Here $C$ is some constant. Since $Y_i$ and $A_i$ are sub-gaussian random variables, we use sub-gaussian concentration to establish that for some constant $c>0$,
		\bse
		|\wh\alpha_{00}-\alpha_{00}|+|\wh\alpha_{10}-\alpha_{10}|\leq c\left\{\sqrt{\frac{\log(1/\delta)}{n_1}}\right\}
		\ese
		with probability at least $1-2\delta$.

		For the second term, we have
		\bse
		&&\left|\frac{\beta_{y0}}{\alpha_{y0}}\frac{\pr(A=0|R=1)\wh\pr(A=0|R=0)-\pr(A=0|R=0)\wh\pr(A=0|R=1)}{\pr(A=0|R=0)\wh\pr(A=0|R=0)}\right|\\
		&\leq&\left|\frac{\beta_{y0}}{\alpha_{y0}}\right|\frac{\pr(A=0|R=1)}{\pr(A=0|R=0)\wh\pr(A=0|R=0)}\left|\{\wh\pr(A=0|R=0)-\pr(A=0|R=0)\}\right|\\
		&&+\left|\frac{\beta_{y0}}{\alpha_{y0}}\right|\frac{\pr(A=0|R=0)}{\pr(A=0|R=0)\wh\pr(A=0|R=0)}\left|\{\pr(A=0|R=1)-\wh\pr(A=0|R=1)\}\right|\\
		&\leq&\left|\frac{\beta_{y0}}{\alpha_{y0}}\right|\frac{p_{1\cdot 0}(1-\pi)n_0}{p_{0\cdot 0}\pi n_{0\cdot 0}}\left|\{\wh\pr(A=0|R=0)-\pr(A=0|R=0)\}\right|\\
		&&+\left|\frac{\beta_{y0}}{\alpha_{y0}}\right|\frac{n_0}{n_{0\cdot 0}}\left|\{\pr(A=0|R=1)-\wh\pr(A=0|R=1)\}\right|\\
		&\leq& CB_1(\left|\{\wh\pr(A=0|R=0)-\pr(A=0|R=0)\}\right|+\left|\{\pr(A=0|R=1)-\wh\pr(A=0|R=1)\}\right|).
		\ese
		Here $C$ is some constant. Since $A_i$ is a sub-gaussian random variable, we use sub-gaussian concentration to establish that for some constant $c>0$,
		\bse
		&&\left|\{\wh\pr(A=0|R=0)-\pr(A=0|R=0)\}\right|+\left|\{\pr(A=0|R=1)-\wh\pr(A=0|R=1)\}\right|\\
		&\leq& c\left\{\sqrt{\frac{\log(1/\delta)}{n_0}}+\sqrt{\frac{\log(1/\delta)}{n_1}}\right\}
		\ese
		with probability at least $1-2\delta$.
	\end{proof}
	
\end{document}